\relax
\documentclass[10pt,twocolumn,letterpaper]{article}
\usepackage{cvpr}  % DO NOT CHANGE THIS
\usepackage{times}  % DO NOT CHANGE THIS
\usepackage{helvet} % DO NOT CHANGE THIS
\usepackage{courier}  % DO NOT CHANGE THIS
\usepackage[hyphens]{url}  % DO NOT CHANGE THIS
\usepackage{graphicx}  % DO NOT CHANGE THIS
\usepackage{amssymb}% http://ctan.org/pkg/amssymb
\usepackage{pifont}% http://ctan.org/pkg/pifont
\newcommand{\xmark}{\ding{55}}%

\usepackage{graphicx} % more modern
\usepackage{subfigure} 

\usepackage[ruled,linesnumbered]{algorithm2e}
\usepackage{helvet}
\usepackage{courier}

\usepackage{makecell}
\usepackage{graphicx}
\usepackage{subfigure}
\usepackage{color}
\usepackage{amsfonts}
\usepackage{amsmath}
\usepackage{amssymb}

\def\etal{{\em et al.}}

\def\mB{{\mathcal B}}

\def\mD{{\mathcal D}}

\def\mH{{\mathcal H}}

\def\mL{{\mathcal L}}

\def\mP{{\mathcal P}}

\def\mS{{\mathcal S}}

\def\mX{{\mathcal X}}
\def\mY{{\mathcal Y}}
\def\mZ{{\mathcal{Z}}}

\DeclareMathAlphabet\mathbfcal{OMS}{cmsy}{b}{n}
\def\0{{\bf 0}}
\def\1{{\bf 1}}

\def\bx{{\bf x}}
\def\by{{\bf y}}
\def\bz{{\bf z}}

\def\mmE{{\mathbb E}}

\def\bx{{\bf x}}

\def\by{{\bf y}}

\def\bz{{\bf z}}

\newtheorem{deftn}{Definition}
\newtheorem{thm}{Theorem}

\newtheorem{remark}{Remark}

\newtheorem{proof}{Proof}
\def\blue{\textcolor{blue}}

\usepackage{times}
\usepackage{epsfig}
\usepackage{graphicx}
\usepackage{amsmath}
\usepackage{amssymb}

\usepackage{subfigure}
\usepackage{booktabs}
\usepackage{multirow}

\usepackage{color}
\usepackage{capt-of}

\usepackage{comment}
\usepackage{array}

\newcommand{\eat}[1]{}

\usepackage[pagebackref=true,breaklinks=true,colorlinks,bookmarks=false]{hyperref}

\cvprfinalcopy % *** Uncomment this line for the final submission

\def\cvprPaperID{1519} % *** Enter the CVPR Paper ID here

\begin{document}

%%%%%%%%% TITLE
% \def\mytitle{Deep Image Super-Resolution with Dual Reconstruction Consistency}
\def\mytitle{Closed-loop Matters: Dual Regression Networks for Single Image Super-Resolution}
\title{Closed-loop Matters: Dual Regression Networks for \\ Single Image Super-Resolution}

\author{
    Yong Guo\thanks{Authors contributed equally.}~, Jian Chen$^*$, Jingdong Wang$^*$, Qi Chen, Jiezhang Cao, Zeshuai Deng,\\ Yanwu Xu\thanks{Corresponding author.}~, Mingkui Tan$^{\dagger}$\\
    South China University of Technology, Guangzhou Laboratory, Microsoft Research Asia, Baidu Inc. \\
    \{guo.yong, sechenqi, secaojiezhang, sedengzeshuai\}@mail.scut.edu.cn, \\
    \{mingkuitan, ellachen\}@scut.edu.cn, jingdw@microsoft.com, ywxu@ieee.org 
    % welleast@outlook.com, ywxu@ieee.org
}

\maketitle

\begin{abstract}
Deep neural networks have exhibited promising performance in image super-resolution (SR) by learning a non-linear mapping function from low-resolution (LR) images to high-resolution (HR) images. However, there are two underlying limitations to existing SR methods. 
{First, learning the mapping function from LR to HR images is typically an ill-posed problem, because there exist infinite HR images that can be downsampled to the same LR image. As a result, the space of the possible functions can be extremely large, which makes it hard to find a good solution.}
Second, 
the paired LR-HR data
may be unavailable in real-world applications and the underlying degradation method is often unknown. For such a more general case, existing SR models often incur the adaptation problem and yield poor performance. To address the above issues, we propose a dual regression scheme by introducing an additional constraint on LR data to reduce the space of the possible functions.
Specifically, besides the mapping from LR to HR images, we learn an additional dual regression mapping estimates the down-sampling kernel and reconstruct LR images, which forms a closed-loop to provide additional supervision.
More critically, since the dual regression process does not depend on HR images, we can directly learn from LR images. In this sense, we can easily adapt SR models to real-world data, \eg, raw video frames from YouTube. Extensive experiments with paired training data and unpaired real-world data demonstrate our superiority over existing methods.
\end{abstract}

\section{Introduction}\label{sec:introduction}

Deep neural networks (DNNs) have been the workhorse of many real-world applications, including image classification~\cite{he2016deep,guo2018double,guo2020multi,guo2019nat,liu2020discrimination,guo2016shallow}, video understanding~\cite{zeng2020dense,zeng2019graph,zeng2019breaking,chen2019relation} and many other applications~\cite{chen2020intelligent,zhang2019whole,zhang2019collaborative,guo2019auto,hu2019multi}.
Recently, image super-resolution (SR) has become an important task that aims at learning a nonlinear mapping to reconstruct high-resolution (HR) images from low-resolution (LR) images. 
Based on DNNs, many methods have been proposed to improve SR performance~\cite{zhang2018image,lim2017enhanced,guo2018dual,guo2020hierarchical,zhang2019ranksrgan}.
However, these methods may suffer from two limitations.

\begin{figure}[t]
	\centering
	\includegraphics[width=0.93\columnwidth]{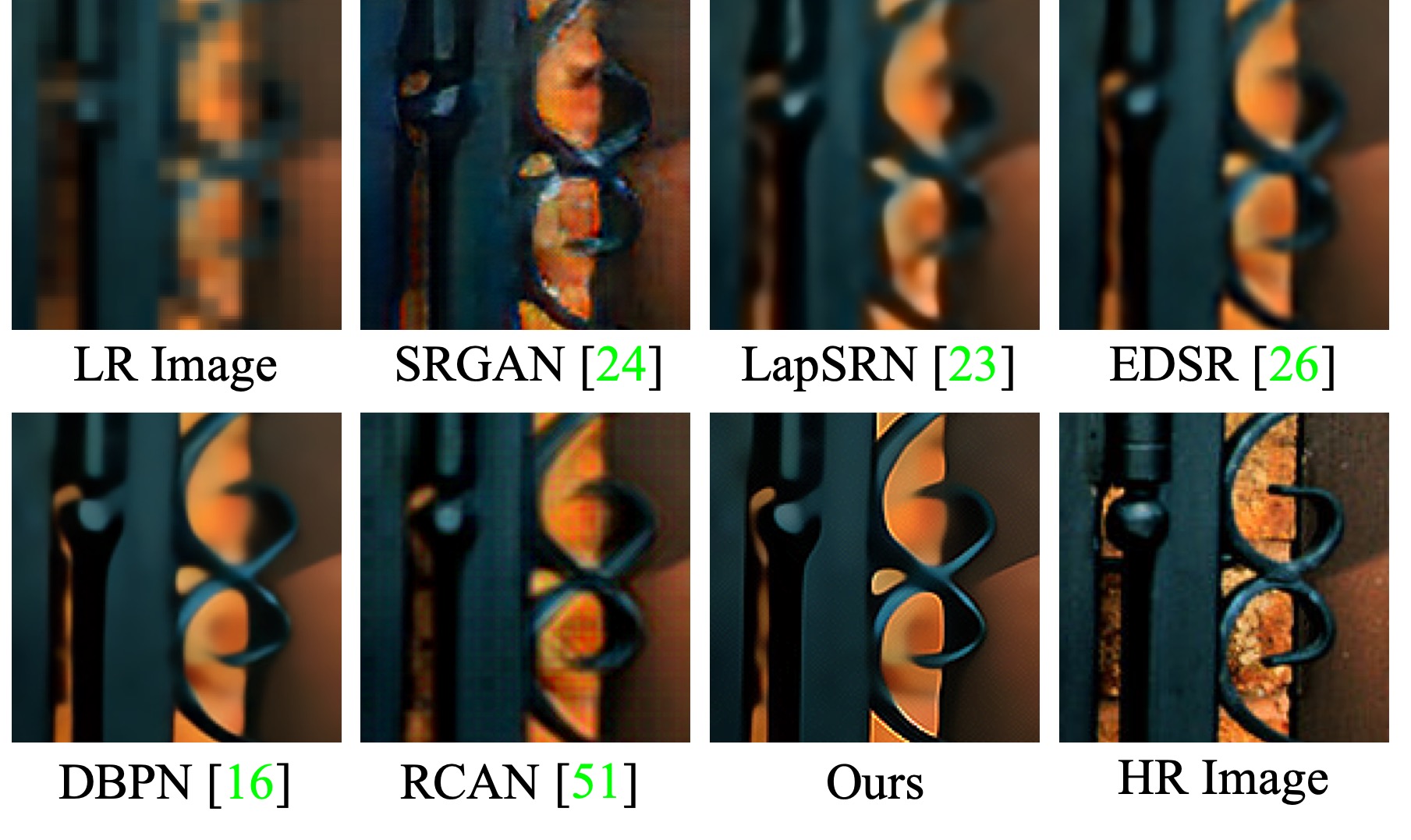}
	\caption{
	    Performance comparison of the images produced by the state-of-the-art methods for $8\times$ SR.
	    Our dual regression scheme is able to produce sharper images than the baseline methods.
	}
	\label{fig:motivation}
\end{figure}

{
First, learning the mapping from LR to HR images is typically an ill-posed problem since there exist infinitely many HR images that can be downscaled to obtain the same LR image~\cite{ulyanov2018deep}.
Thus, the space of the possible functions that map LR to HR images becomes extremely large. {As a result, the learning performance can be limited since learning a good solution in such a large space is very hard.} 
To improve the SR performance, one can design effective models by increasing the model capacity, \eg, EDSR~\cite{lim2017enhanced}, DBPN~\cite{DBPN2018}, and RCAN~\cite{zhang2018image}. 
{However, these methods still suffer from the large space issue of possible mapping functions, resulting in the limited performance without producing sharp textures}~\cite{ledig2016photo} (See Figure~\ref{fig:motivation}). 
Thus, how to reduce the possible space of the mapping functions to improve the training of SR models becomes an important problem.
}

Second, {it is hard to obtain a promising SR model when the paired data are unavailable~\cite{yuan2018unsupervised,zhao2018unsupervised}. Note that most SR methods rely on the paired training data, \ie, HR images with their Bicubic-degraded LR counterparts. However, the paired data may be unavailable and the unpaired data often dominate the real-world applications. Moreover, the real-world data do not necessarily have the same distribution to the LR images obtained by a specific degradation method (\eg, Bicubic). Thus, learning a good SR model for real-world applications can be very challenging.}
{More critically, if we directly apply existing SR models to real-world data, they often incur a severe adaptation problem and yield poor performance~\cite{yuan2018unsupervised,zhao2018unsupervised}.}
{Therefore, how to effectively exploit the unpaired data to adapt SR models to real-world applications becomes an urgent and important problem.}

{In this paper, we propose a novel dual regression scheme {that forms a closed-loop} to enhance SR performance.}
{To address the first limitation, we introduce an additional constraint to reduce the possible space such that the super-resolved images can reconstruct the input LR images.}
{Ideally, if the mapping from LR$\to$HR is optimal, the super-resolved images can be downsampled to obtain the same input LR image.
With such a constraint, we are able to estimate the underlying downsampling kernel and hence reduce the space of possible functions to find a good mapping from LR to HR (See theoretical analysis in Remark~\ref{alg:remark1}). Thus, it becomes easier to obtain promising SR models (See the comparison in Figure~\ref{fig:motivation}).
}
{To address the second limitation, since the regression of LR images does not depend on HR images, our method can directly learn from the LR images.}
In this way, we can easily adapt SR models to the real-world LR data, \eg, raw video frames from Youtube.
Extensive experiments on the SR tasks with paired training data and unpaired real-world data demonstrate the superiority of our method over existing methods.

Our contributions are summarized as follows:
\begin{itemize}
	\item We develop a dual regression scheme by introducing an additional constraint {such that the mappings can form a closed-loop and LR images can be reconstructed to enhance the performance of SR models}. 
	Moreover, we also theoretically analyze the generalization ability of the proposed scheme, which further confirms its superiority to existing methods.
	\item We study a more general super-resolution case {where there is no corresponding HR data w.r.t. the real-world LR data}.
	With the proposed dual regression scheme, deep models can be easily adapted to real-world data, \eg, raw video frames from YouTube.
	\item Extensive experiments on both the SR tasks with paired training data and unpaired real-world data demonstrate the effectiveness of the proposed dual regression scheme in image super-resolution.
\end{itemize}

\section{Related Work}

\noindent\textbf{Supervised super-resolution.}
Many efforts have been made to improve the performance of SR, including the interpolation-based approaches~\cite{hou1978cubic} and reconstruction-based methods~\cite{DBPN2018,li2019feedback,zhang2018image}.
Haris~\etal ~\cite{DBPN2018} propose a back-projection network (DBPN) that consists of several up- and down-sampling layers to iteratively produce LR and HR images.
Zhang~\etal ~\cite{zhang2018image} propose the channel attention mechanism to build a deep model called RCAN to further improve the performance of SR. 
However, these methods still have a very large space of the possible mappings which makes it hard to learn a good solution.

\begin{figure}[t]
	\centering
	\includegraphics[width=0.95\columnwidth]{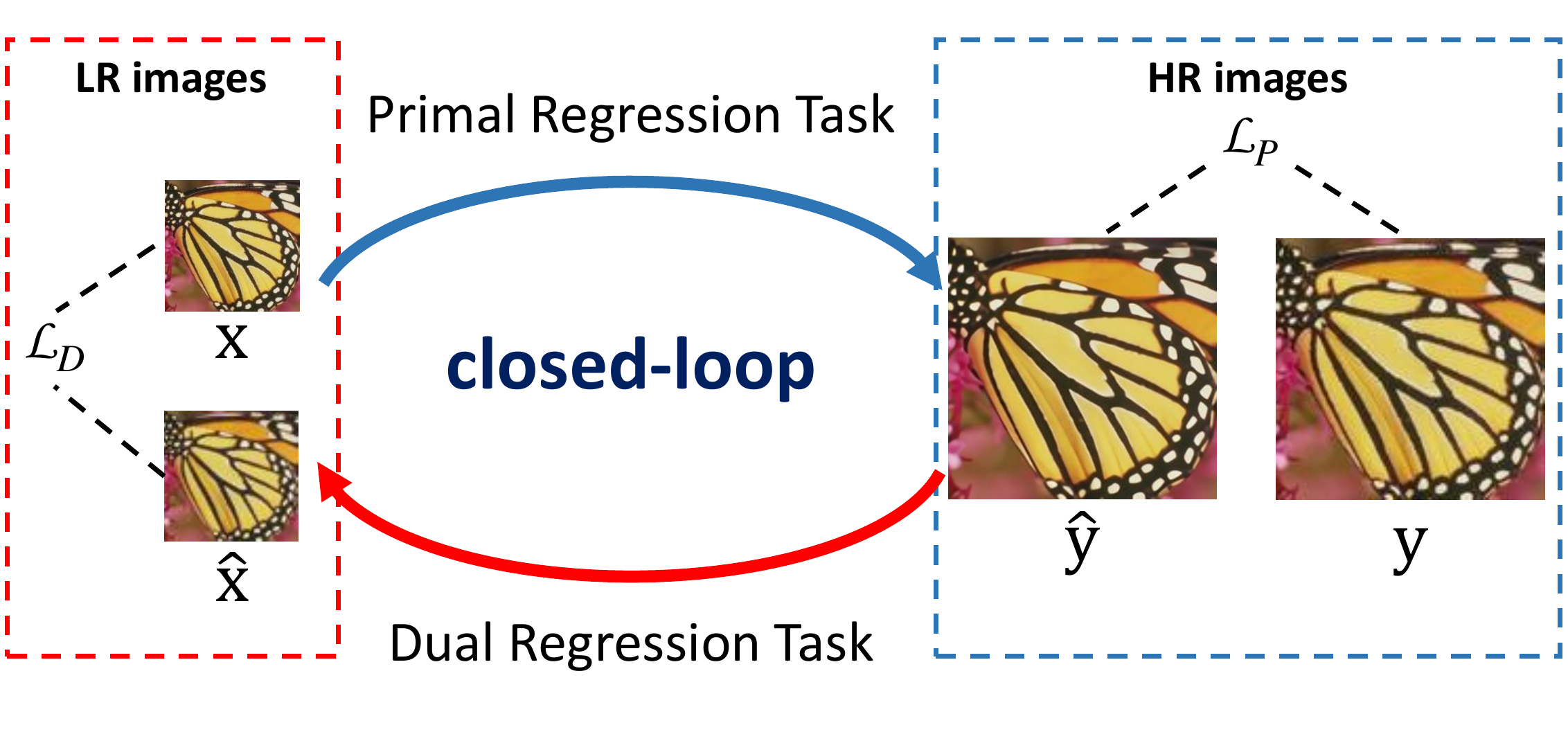}
	\vspace{-5pt}
	\caption{
		Dual regression training scheme, which contains a primal regression task for super-resolution and a dual regression task to project super-resolved images back to LR images. The primal and dual regression tasks form a closed-loop.
	}
	\vspace{-10pt}
	\label{fig:dual_connection}
\end{figure}

\vspace{5pt}
\noindent\textbf{Unsupervised super-resolution.}
There is an increasing interest in learning super-resolution models without paired data in the unsupervised setting~\cite{yuan2018unsupervised, zhao2018unsupervised}.
Based on CycleGAN~\cite{zhu2017unpaired}, Yuan \etal~\cite{yuan2018unsupervised} propose a CinCGAN model to generate HR images without paired data. 
{Recently, some blind SR methods~\cite{ bell2019blind, zhou2019kernel} were proposed to learn the unknown degradation methods.}
However, these methods often totally discard the paired synthetic data, {which can} be obtained very easily and used to boost the training. 
{On the contrary, our dual regression scheme seeks to adapt SR models to new LR data by exploiting both the real-world LR data and the paired synthetic data.}

\vspace{5pt}
\noindent\textbf{Dual learning.}
{
Dual learning methods~\cite{he2016dual,pmlr-v70-xia17a,xia2018model,zhang2018deep}
contain a primal model and a dual model to learn two opposite mappings simultaneously to enhance the performance of language translation.
Recently, this scheme has also been used to perform image translation without paired training data, \eg, CycleGAN~\cite{zhu2017unpaired} and DualGAN~\cite{yi2017dualgan}.}
{Specifically, a cycle consistency loss is proposed to avoid the mode collapse issue of GAN methods~\cite{ zhu2017unpaired,cao2018adversarial,cao2019multi} and help minimize the distribution divergence. However, these methods cannot be directly applied to the standard SR problem.
By contrast, we use the closed-loop to reduce the space of possible functions of SR. Moreover, we consider learning asymmetric mappings and provide a theoretical guarantee on the rationality and necessity of using a cycle.
}

\section{Proposed Method}
{
We propose a dual regression scheme to deal with both the {paired} and {unpaired} training data for super-resolution (SR).
The overall training scheme is shown in Figure \ref{fig:dual_connection}.
}

\subsection{Dual Regression Scheme for Paired Data}\label{sec:DRN}

{Existing methods only focus on learning the mapping from LR to HR images. 
However, 
the space of the possible mapping functions can be extremely large, making the training very difficult.
To address this issue, we propose a dual regression scheme by introducing an additional constraint on LR data. 
Specifically, besides learning the mapping LR$\to$ HR, we also learn an inverse/dual mapping from the super-resolved images back to the LR images.
}

Let $\bx \in \mathcal{X}$ be LR images and $\by \in \mathcal{Y}$ be HR images.
We simultaneously learn the primal mapping $P$ to reconstruct HR images and the dual mapping $D$ to reconstruct LR images.
Note that the dual mapping can be regarded as the estimation of the underlying downsampling kernel.
Formally, we formulate the SR problem into the dual regression scheme which involves two regression tasks.
\begin{deftn} \textbf{\emph{(Primal Regression Task) }}\label{definition: primal_model}
	We seek to find a function $P$: $\mX {\rightarrow} \mY$, such that the prediction $P(\bx)$ is similar to its corresponding HR image $\by$.
\end{deftn}
\begin{deftn} \textbf{\emph{({Dual} Regression Task) }}\label{definition: dual_model}
	We seek to find a function $D$: $\mY {\rightarrow} \mX$, such that the prediction of $D(\by)$ is similar to the original input LR image $\bx$.
\end{deftn}

The primal and dual learning tasks can form a {closed-loop} and provide informative supervision to train the models $ P $ and $ D $.
If $P(\bx)$ was the correct HR image, then the down-sampled image $D(P(\bx))$ should be very close to the input LR image $\bx$.
With this constraint, we can reduce the function space of possible mappings and make it easier to learn a better mapping to reconstruct HR images. 
{To verify this, we provide a theoretical analysis in Section~\ref{sec:analysis}.}

By jointly learning these two learning tasks, we propose to train the super-resolution models as follows.
Given a set of $N$ paired samples $\mS_P = \left\{ (\bx_i, \by_i) \right\}_{i=1}^N$,
where $\bx_i$ and $\by_i$ denote the $i$-th pair of low- and high-resolution images in the set of paired data.
The training loss can be written as
\begin{equation}
\label{eq:dual_regression}
\sum_{i=1}^N \underbrace{\mL_P \Big( P(\bx_i), \by_i \Big)}_{\rm primal~regression~loss} + ~~\lambda \underbrace{\mL_D \Big( D(P(\bx_i)), \bx_i \Big)}_{\rm dual~regression~loss},
\end{equation}
{where $\mL_P$ and $\mL_D$ denote the loss function ($\ell_1$-norm) for the primal and dual regression tasks, respectively.}
Here, $\lambda$ controls the weight of the dual regression loss (See the sensitivity analysis of $\lambda$ in Section~\ref{exp:lambda}).

{
Actually, we can also add a constraint on the HR domain, \ie, downscaling and upscaling to reconstruct the original HR images.
However, it greatly increases the computation cost (approximately $2\times$ of the original SR model) and the performance improvement is very limited (See results in supplementary).
In practice, we only add the dual regression loss on LR data, which significantly improves the performance while preserving the approximately the same cost to the original SR model (See discussions in Section~\ref{sec:architecture}).
}

\subsection{Dual Regression for Unpaired Data}\label{sec:blind} 
{We consider a more general SR case where there is no corresponding HR data w.r.t. the real-world LR data. 
More critically, the degradation methods of LR images are often unknown, making this problem very challenging.
In this case, existing SR models often incur the severe adaptation problem~\cite{yuan2018unsupervised,zhao2018unsupervised}.} 
To alleviate this issue, 
we propose an efficient algorithm to adapt SR models to the new LR data. The training algorithm is shown in Algorithm~\ref{alg:semisupervised}.

\begin{algorithm}[t]
	\label{alg:semisupervised}
	\caption{\small Adaptation Algorithm on Unpaired Data.}
	\KwIn{Unpaired real-world data: $\mS_U$;\\
	    ~~~~~~~~~~~~Paired synthetic data: $\mS_P$;\\
	    ~~~~~~~~~~~~Batch sizes for $\mS_U$ and $\mS_P$: $m$ and $n$;\\
	    ~~~~~~~~~~~~Indicator function: ${\bf{1}}_{\mS_P}(\cdot)$.\\
	}
	Load the pretrained models $P$ and $D$.
	
	\While{not convergent}{
	    Sample unlabeled data $\{ \bx_i \}_{i=1}^{m}$ from $\mS_U$;\\
		Sample labeled data $\{ (\bx_i, \by_i) \}_{i=m+1}^{m+n}$ from $\mS_P$;\\
		// \emph{Update the primal model} \\
		Update $P$ by minimizing the objective:\\
		~~~~~{\small $\sum\limits_{i=1}^{m+n} {\bf{1}}_{\mS_P}(\bx_i) \mL_P\Big( P(\bx_i), \by_i \Big) {+} \lambda \mL_D\Big( D(P(\bx_i)), \bx_i \Big)$} \\
		// \emph{Update the dual model} \\
		Update $D$ by minimizing the objective:\\
		~~~~~~~~~~~~~~~~~~~~{\small $\sum\limits_{i=1}^{m+n} \lambda \mL_D\Big( D(P(\bx_i)), \bx_i \Big)$}  \\
	}
\end{algorithm}

\begin{figure*}[t]
	\centering
	\includegraphics[width=1.8\columnwidth]{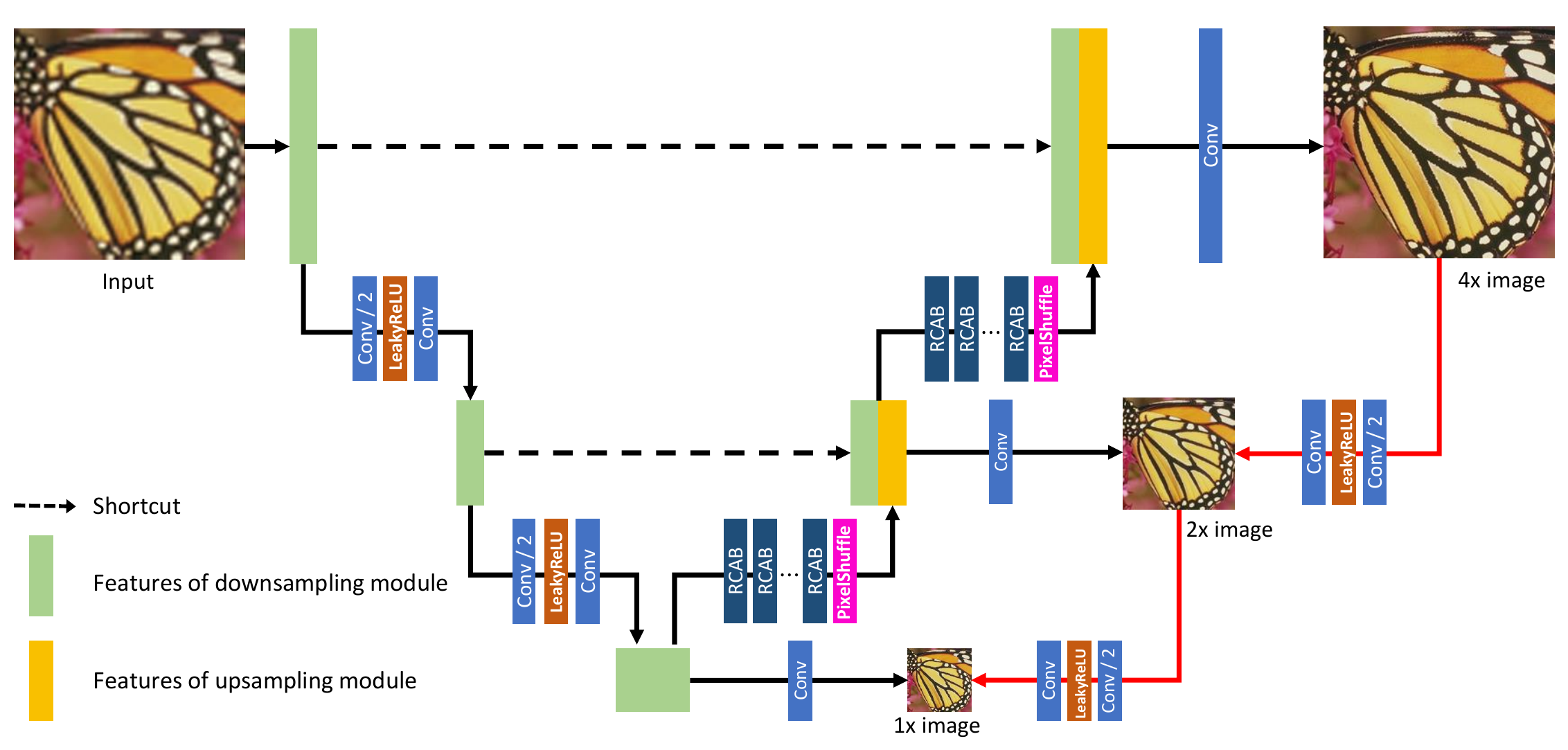}
	\caption{
		The architecture of DRN for 4${\times}$ SR. {DRN} contains a primal network and a dual network (marked as red lines). The green box denotes the feature maps of the downsampling module (left half) while the yellow box refers to the feature maps of the upsampling module (right half). Following U-Net, we concatenate the corresponding shallow and deep feature maps via shortcut connections.
	}
	\label{fig:primal_dual_model}
\end{figure*}

{Note that the dual regression mapping learns the underlying degradation methods and does not necessarily depend on HR images. Thus, we can use it to directly learn from the unpaired real-world LR data to perform model adaptation.}
{To ensure the reconstruction performance of HR images, we also incorporate the information from paired synthetic data that can be obtained very easily (\eg, using the Bicubic kernel).}
{Given $M$ unpaired LR samples and $N$ paired synthetic samples, the objective function can be written as:}
\begin{equation}\label{eq:adaptaion_ojb}
\sum_{i=1}^{M{+}N} {\bf{1}}_{\mS_P}(\bx_i) \mL_P \Big( P(\bx_i), \by_i \Big) + \lambda \mL_D \Big( D(P(\bx_i)), \bx_i \Big),
\end{equation}
where ${\bf{1}}_{\mS_P}(\bx_i)$ is an indicator function that equals 1 when $\bx_i \in \mS_P$, {and otherwise the function equals 0}.

\subsection{Training Method}
\noindent \textbf{Training method on paired data.}
{
Given paired training data, we follow the learning scheme of supervised SR methods~\cite{DBPN2018,lim2017enhanced} and
train model by minimizing Eqn.~(\ref{eq:dual_regression}). More details are shown in Section~\ref{sec:exp} and the supplementary.
}

\noindent \textbf{Training method on unpaired data.}
{As shown in Algorithm~\ref{alg:semisupervised}, for each iteration, we first sample $m$ unpaired real-world data from $\mS_U$ and $n$ paired synthetic data from $\mS_P$, respectively. Then, we train our model end-to-end by minimizing the objective in Eqn.~(\ref{eq:adaptaion_ojb}).  
For convenience, we define the data ratio of unpaired data as 
\begin{equation}\label{eq:rho}
\rho = m / (m + n).
\end{equation}
Since paired synthetic data can be obtained very easily (\eg, performing Bicubic kernel to produce LR-HR pairs), we can adjust $\rho$ by changing the number of paired synthetic samples $n$. In practice, we set $\rho=30\%$ to obtain the best results (See the discussions in Section~\ref{exp:amount}).}
{With the proposed dual regression scheme, we can adapt SR models to the various unpaired data while preserving good reconstruction performance (See results in Section~\ref{sec:semi_results}).
}

\subsection{Differences from CycleGAN based SR Methods}
There are several differences and advantages of {DRN} compared to CycleGAN based SR methods. 
{First, 
CycleGAN based methods~\cite{yuan2018unsupervised,zhu2017unpaired} use a cycle consistency loss to avoid the possible mode collapse issue when solving the under-constrained image translation problem~\cite{zhu2017unpaired}.
Unlike these methods, we seek to improve the performance of our SR model by adding an extra constraint, which reduces the possible function space by mapping the SR images back to the corresponding LR images.}
Second, CycleGAN based methods totally discard the paired synthetic data, which, however, can be obtained very easily.
On the contrary, our DRN simultaneously exploits both paired synthetic data and real-world unpaired data to enhance the training.

\section{More Details}\label{sec:more_discussion}
{In this section, we first depict the architecture of our dual regression network (DRN).
Then, we conduct a theoretical analysis to justify the proposed dual regression scheme.}

\subsection{Architecture Design of DRN}\label{sec:architecture}
We build our {DRN} upon the design of U-Net for super-resolution~\cite{iqbal2019super,ronneberger2015u} (See Figure~\ref{fig:primal_dual_model}).
Our {DRN} model consists of two parts: a primal network and a dual network. We present the details for each network as follows. 

The primal network follows the downsampling-upsampling design of U-Net. Both the downsampling (left half of Figure~\ref{fig:primal_dual_model}) and upsampling (right half of Figure~\ref{fig:primal_dual_model}) modules contain $\log_2(s)$ basic blocks, where $s$ denotes the scale factor. This implies that the network will have 2 {blocks} for 4$\times$ upscaling (See Figure~\ref{fig:primal_dual_model}) and 3 blocks for 8$\times$ upscaling.
Unlike the baseline U-Net, we build each basic block using $B$ residual channel attention block (RCAB)~\cite{zhang2018image} to improve the model capacity.
Following~\cite{wang2015deep,lai2017deep}, we add additional outputs to produce images at the corresponding scale (\ie, $1\times$, $2\times$, and $4\times$ images) and apply the proposed loss to them to train the model. 
{Note that we use the Bicubic kernel to upscale LR images before feeding them into the primal network.}
Please refer to the supplementary for more details.

We design a dual network to produce the down-sampled LR images from the super-resolved ones (See red lines in Figure~\ref{fig:primal_dual_model}).
Note that the dual task aims to learn a downsampling operation, which is much simpler than the primal task for learning the upscaling mapping.
Thus, we design the dual model with only two convolution layers and a LeakyReLU activation layer~\cite{maas2013rectifier}, which has much lower computation cost than the primal model but works well in practice (See results in Section~\ref{sec:exp}). 

\subsection{Theoretical Analysis}\label{sec:analysis}
We theoretically analyze the generalization bound for the proposed dual regression scheme on paired data. Since the case with unpaired data is more complicated, we will investigate the theoretical analysis method in the future.
Due to the page limit, all the definitions, proofs, and lemmas are put in the supplementary.

The generalization error of the 
{dual regression scheme}
is to measure how accurately the algorithm predicts the unseen test data in the primal and dual tasks. Let $E(P, D)=\mmE[\mL_P(P(\bx), \by) {+} {\lambda}\mL_D (D(P(\bx)), \bx)] $ and $\hat{E}(P, D)$ is its empirical loss, we obtain a generalization bound of the proposed model using Rademacher complexity \cite{mohri2012foundations}.

\begin{thm} \label{theorem: generalization bound}
	Let $ \mL_P(P(\bx), \by) {+} {\lambda}\mL_D (D(P(\bx)), \bx) $ be a mapping from $ \mX {\times} \mY $ to $ [0, C] $ with the upper bound $C$, and the function space $ \mH_{dual} $ be infinite. Then, for any {error} $ \delta {>} 0 $, with probability at least $ 1{-}\delta $, the generalization error $E(P, D)$ (\ie, expected loss)  satisfies for all $ (P, D) {\in} \mH_{dual} $:
	\begin{align}
	E(P, D) \leq \hat{E}(P, D) {+} 2 \hat{R}_{\mZ}^{DL}(\mH_{dual}) {+} 3C \sqrt{\frac{1}{2N} \log\left(\frac{1}{\delta}\right)}, \nonumber
	\end{align}
	where $N$ is the number of samples and $\hat{R}_{\mZ}^{DL}$ is the empirical Rademacher complexity of dual learning. 
	Let $\mB(P, D)$ be the generalization bound of the dual regression SR, \ie $\mB(P, D){=}2\hat{R}_{\mZ}^{DL}(\mH_{dual}) {+} 3C \sqrt{\frac{1}{2N} \log\left(\frac{1}{\delta}\right)}$, we have
	\begin{align*}
	    \mB(P, D) \leq \mB(P),
	\end{align*}
	where $\mB(P), P{\in}\mH$ is the generalization bound of the supervised learning \textit{w.r.t.} the Rademacher complexity $\hat{R}_{\mZ}^{SL}(\mH)$.
\end{thm}
This theorem shows the generalization bound of the {dual regression} scheme relies on the Rademacher complexity of a function space $\mH_{dual}$. 
{From} Theorem~\ref{theorem: generalization bound}, the {dual regression SR scheme} has a smaller generalization bound {than traditional SR method}, and thus it helps to achieve more accurate SR predictions. More discussions can be referred to Remark \ref{alg:remark1}.
We highlight that the derived generalization bound of the {dual regression scheme}, where the loss function is bounded by $ [0, C] $, is more general than \cite{pmlr-v70-xia17a}.
{Moreover, this generalization bound is tight when training data is sufficient, and the primal and dual models are powerful enough.}

\begin{figure}[t]
    \centering
	\subfigure[Performance vs. model size for $4 \times$ SR.]{
		\includegraphics[width = 0.96\columnwidth]{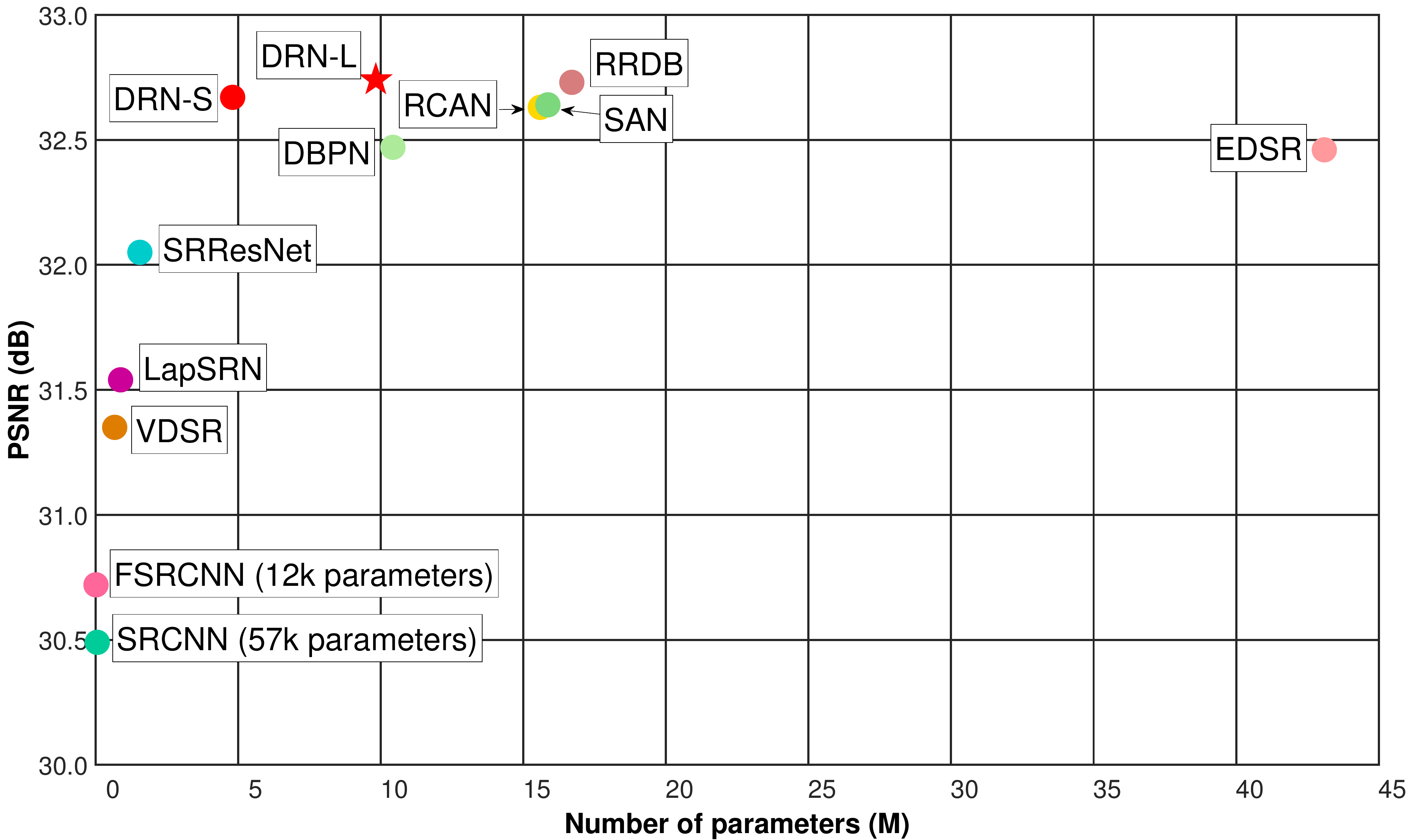}
	}
	\subfigure[Performance vs. model size for $8 \times$ SR.]{
		\includegraphics[width = 0.96\columnwidth]{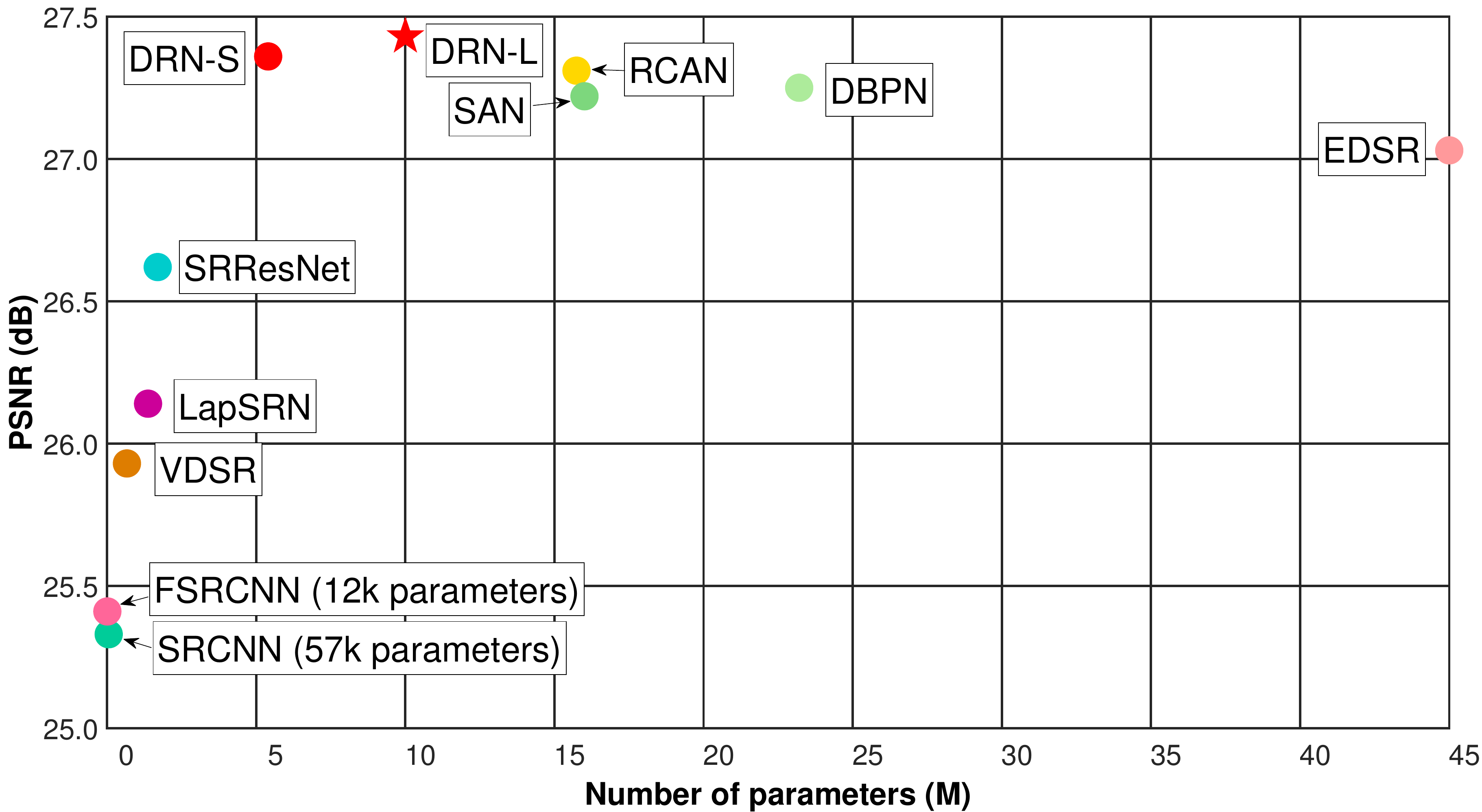}
	}
    \caption{
    Comparisons of the performance and the number of parameters among different $4\times$ SR models on the Set5 dataset.
    }
    \label{fig:complexity_main}
    \vspace{-10 pt}
\end{figure}

\begin{remark}\label{alg:remark1}
	Based on the definition of the Rademacher complexity, the capacity of the function space $ \mH_{dual} \small{\in} \mP \small{\times} \mD $ is smaller than the capacity of function space $ \mH \small{\in} \mP $ or $ \mH \small{\in} \mD $ in traditional supervised learning, \ie, $ \hat{R}_{\mZ}^{DL} \leq \hat{R}_{\mZ}^{SL} $, where $ \hat{R}_{\mZ}^{SL} $ is the Rademacher complexity defined in supervised learning. In other words, the dual regression scheme has smaller generalization bound than the primal feed-forward scheme and the proposed dual regression scheme helps the primal model to achieve more accurate SR predictions.
\end{remark}

\section{Experiments}\label{sec:exp}
{We extensively evaluate the proposed method on the image super-resolution tasks under the scenarios with paired Bicubic data and unpaired real-world data. All implementations are based on PyTorch.\footnote{The source code is available at \href{https://github.com/guoyongcs/DRN}{https://github.com/guoyongcs/DRN}.}}

\begin{table*}[t]
	\centering
	\caption{Performance comparison with state-of-the-art algorithms for 4$\times$ and 8$\times$ image super-resolution. The \textbf{bold} number indicates the best result and the \blue{\textbf{blue}} number indicates the second best result. ``-'' denotes the results that are not reported.
	}
 	\resizebox{1\textwidth}{!}
	{
		{
			\begin{tabular}{c|c|c|ccccc}
				\toprule
				\multicolumn{1}{c|}{\multirow{2}[0]{*}{Algorithms}} &
				\multicolumn{1}{c|}{\multirow{2}[0]{*}{Scale}} & \multicolumn{1}{c|}{\multirow{2}[0]{*}{\#Params (M)}}  & \multicolumn{1}{c}{Set5} & \multicolumn{1}{c}{Set14} & \multicolumn{1}{c}{BSDS100} & \multicolumn{1}{c}{Urban100} & \multicolumn{1}{c}{Manga109} \\
				& & & \multicolumn{1}{c}{PSNR / SSIM} & \multicolumn{1}{c}{PSNR / SSIM} & \multicolumn{1}{c}{PSNR / SSIM} & \multicolumn{1}{c}{PSNR / SSIM} & \multicolumn{1}{c}{PSNR / SSIM} \\
				\hline
				Bicubic  & \multirow{13}[0]{*}{4} & - &  28.42 / 0.810    &    26.10 / 0.702    &   25.96 / 0.667    &   23.15 / 0.657    &   24.92 / 0.789 \\
				ESPCN~\cite{shi2016real}   &   &  -  &  29.21 / 0.851    &   26.40 / 0.744    &   25.50 / 0.696    &   24.02 / 0.726    &   23.55 / 0.795 \\
				SRResNet~\cite{ledig2016photo}  & & 1.6 &   32.05 / 0.891    &   28.49 / 0.782    &   27.61 / 0.736    &   26.09 / 0.783    &   30.70 / 0.908\\
				SRGAN~\cite{ledig2016photo}    &  & 1.6 &   29.46 / 0.838    &   26.60 / 0.718    &   25.74 / 0.666     &   24.50 / 0.736    &   27.79 / 0.856 \\
				LapSRN~\cite{lai2017deep}   &      & 0.9  & 31.54 / 0.885    &   28.09 / 0.770    &    27.31 / 0.727   &   25.21 / 0.756    &   29.09 / 0.890 \\
				SRDenseNet~\cite{tong2017image}   & & 2.0 &   32.02 / 0.893    &   28.50 / 0.778    &   27.53 / 0.733    &   26.05 / 0.781   &    29.49 / 0.899 \\
				EDSR~\cite{lim2017enhanced}   &   &  43.1 &  32.48 / 0.898    &	28.81 / 0.787    & 	27.72 / 0.742     &	  26.64 / 0.803   & 31.03 / 0.915    \\
				DBPN~\cite{DBPN2018}   &    &   10.4   &  32.42 / 0.897 &	28.75 / 0.786     &	   27.67 / 0.739  &	26.38 / 0.794   &	30.90 / 0.913    \\
				RCAN~\cite{zhang2018image} &  &  15.6  & {{32.63}} / {{0.900}} & {{28.85}} / {{0.788}} & {{27.74}} / {{0.743}} & {{26.74}} / {{0.806}} & {{31.19}} / {{0.917}} \\
				SAN~\cite{dai2019second} &   &  15.9  & {{32.64}} / {{0.900}} &	{{28.92}} / {{0.788}} &	27.79 / {{0.743}} &	{{26.79}} / {{0.806}} &	31.18 / 0.916    \\
				RRDB~\cite{wang2018esrgan} &   &  16.7  & \blue{\textbf{32.73}} / 0.901 &	\blue{\textbf{28.97}} / {{0.790}} &	{\textbf{27.83}} / {\textbf{0.745}} &	\blue{\textbf{27.02}} / {\textbf{0.815}} &	 \blue{\textbf{31.64}} / 0.919  \\
				DRN-S &   &  4.8  &  {{32.68}}    /   \blue{\textbf{0.901}}    &   {{28.93}}    /   \blue{\textbf{0.790}}    &   {{27.78}}    /   0.744    &   {26.84}   /   {{0.807}}    &    {{31.52}}   /    \blue{\textbf{0.919}}  \\
				DRN-L &   & 9.8 & {\textbf{32.74}}    /   {\textbf{0.902}}   & {\textbf{28.98}}    /   {\textbf{0.792}}  & {\textbf{27.83}}    /   {\textbf{0.745}} & {\textbf{27.03}}    /   \blue{\textbf{0.813}} & {\textbf{31.73}}    /   {\textbf{0.922}} \\
				\hline
				Bicubic & \multirow{12}[0]{*}{8}  & - & 24.39  /   0.657    &    23.19  / 0.568    &   23.67  /  0.547    &   20.74  /  0.515    &   21.47   / 0.649 \\
				ESPCN~\cite{shi2016real}   &      &  - &  25.02 / 0.697   &   23.45 / 0.598    &   23.92 / 0.574    &   21.20 / 0.554    &  22.04 / 0.683  \\
				SRResNet~\cite{ledig2016photo}   &   & 1.7 &   26.62 / 0.756    &   24.55 / 0.624    &   24.65 /  0.587    &   22.05 /  0.589    &   23.88 / 0.748 \\
				SRGAN~\cite{ledig2016photo}  &    & 1.7  &   23.04 /  0.626    &    21.57 /  0.495    &   21.78  /  0.442    &  19.64  /   0.468    &    20.42  / 0.625 \\
				LapSRN~\cite{lai2017deep}    &   &  1.3 &    26.14  /  0.737    &   24.35  /   0.620    &    24.54 /   0.585   &   21.81  /  0.580    &   23.39   / 0.734 \\
				SRDenseNet~\cite{tong2017image}   &   &  2.3  &   25.99  /  0.704    &   24.23  /  0.581    &   24.45  /   0.530   &    21.67  /   0.562   &   23.09  / 0.712 \\
				EDSR~\cite{lim2017enhanced}   &   &  45.5  &    27.03 / 0.774 & 	25.05 / 0.641 &	24.80 / 0.595 &	22.55 / 0.618  &	24.54 / 0.775  \\
				DBPN~\cite{DBPN2018} &  & 23.2 & 27.25 / 0.786  & 	25.14  /  {{0.649}} &	24.90  /  0.602	 &  22.72  /  0.631 	&  25.14  /  {{0.798}} \\
				RCAN~\cite{zhang2018image} & & 15.7 & {{27.31}}  / {{0.787}}     & {{25.23}}    /   {{0.651}}    &    {{24.96}}    / {{0.605}}    &   \blue{\textbf{22.97}}     / \blue{\textbf{0.643}}     & {{25.23}}   / {{0.802}}  \\
				SAN~\cite{dai2019second} &  & 16.0 & 27.22 / 0.782 &	25.14 / 0.647 &	24.88 / 0.601 &	22.70 / 0.631 &	24.85 / 0.790    \\
				DRN-S &    & 5.4 & \blue{\textbf{27.41}}    /   \blue{\textbf{0.790}}    &   \blue{\textbf{25.25}}    /   \blue{\textbf{0.652}}    &   \blue{\textbf{24.98}}    /   \blue{\textbf{0.605}}    &   {{22.96}}    /   {{0.641}}    &    \blue{\textbf{25.30}}   /    \blue{\textbf{0.805}}  \\
				DRN-L &    & 10.0 &  {\textbf{27.43}}  /   {\textbf{0.792}} &   {\textbf{25.28}}  /   {\textbf{0.653}} &   {\textbf{25.00}}  /   {\textbf{0.606}} &   {\textbf{22.99}}  /   {\textbf{0.644}} &   {\textbf{25.33}}  /   {\textbf{0.806}} \\
				\bottomrule
			\end{tabular}
	}}
	\label{exp:2x_4x_8xsr}
\end{table*}

\begin{figure*}[t]
	\centering
	\subfigure[Visual comparison for $4\times$ super-resolution.]{
		\includegraphics[width = 1\columnwidth]{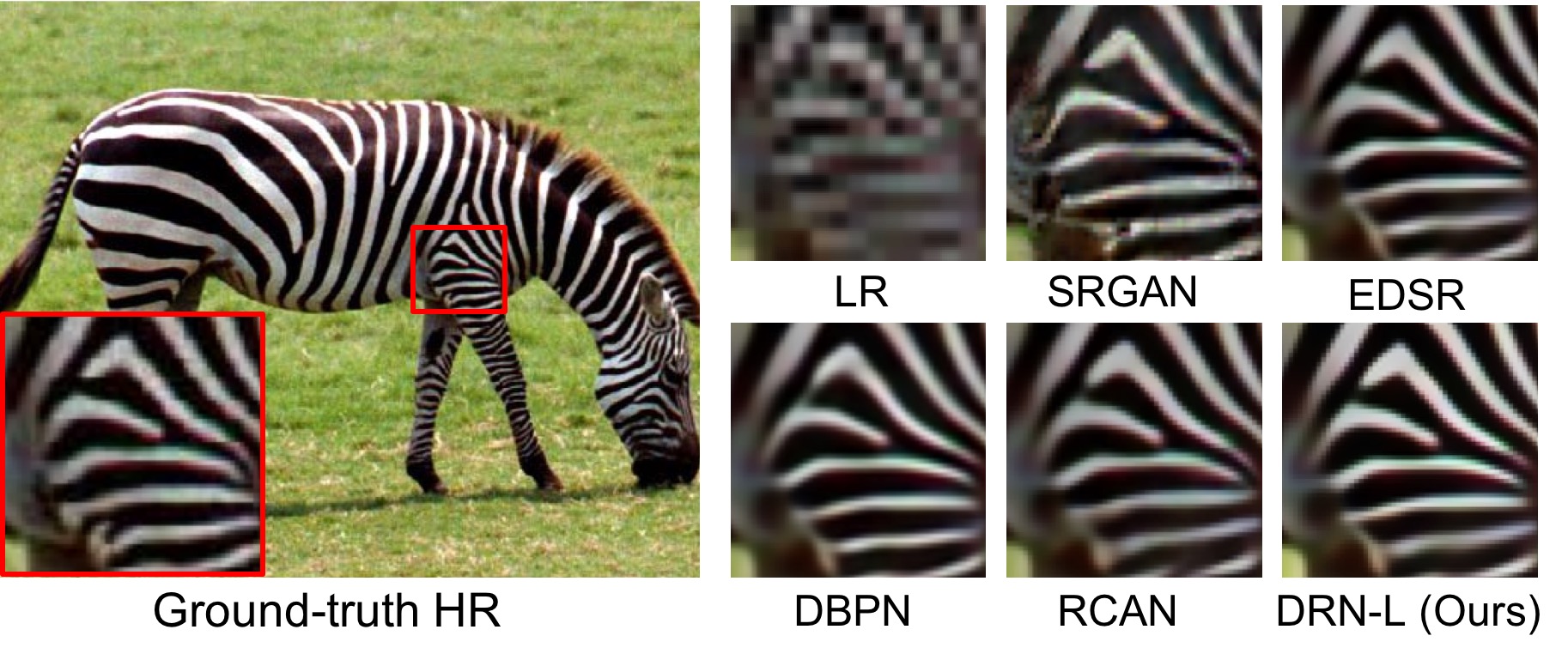}\label{fig:img4x_compare1}
	}
	\subfigure[Visual comparison for $8\times$ super-resolution.]{
		\includegraphics[width = 1\columnwidth]{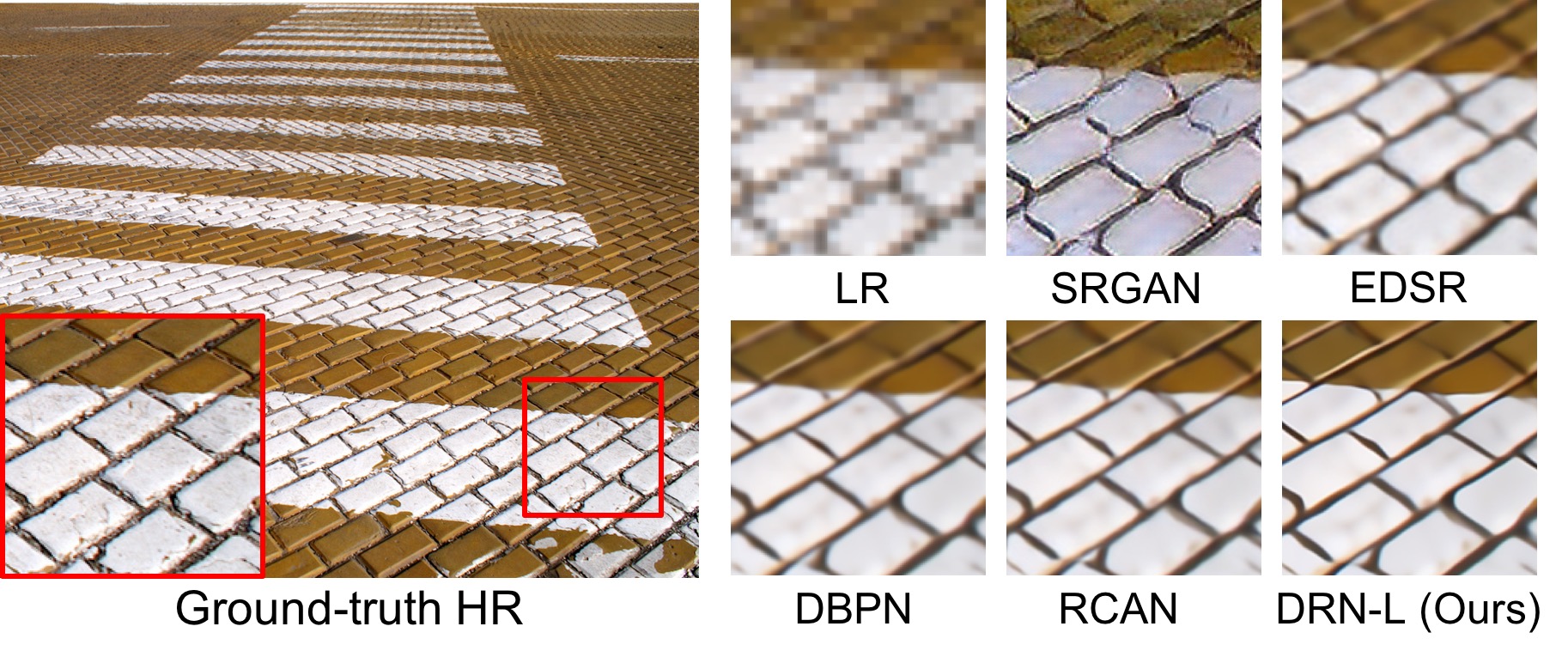}\label{fig:img8x_compare1}
	}	
	\caption{Visual comparison of different methods for (a) $4\times$ and (b) $8\times$ image super-resolution.}
	\label{fig:image_compare}
\end{figure*}

\subsection{Results on Supervised Image Super-Resolution}
In this section, 
we first show an illustrated comparison in terms of performance and model size for $4 \times$ and $8 \times$ SR in Figure~\ref{fig:complexity_main}. Then, we provide a detailed comparison for $4\times$ and $8\times$ SR.
In the experiments, we propose two models, namely a small model DRN-S and a large model DRN-L. 
We obtain the results of all the compared methods from their pretrained models, released code, or their original paper.

\subsubsection{Datasets and Implementation Details}
\vspace{-5 pt}
We compare different methods on five benchmark datasets, including SET5~\cite{DBLP:conf/bmvc/BevilacquaRGA12}, SET14~\cite{zeyde2010single}, BSDS100~\cite{arbelaez2011contour}, URBAN100~\cite{huang2015single} and MANGA109~\cite{matsui2017sketch}. 
Two commonly used image quality metrics are adopted as the metrics, such as \emph{PSNR} and \emph{SSIM}~\cite{wang2004image}.
{Following~\cite{wang2018esrgan}, }
we train our models on {DIV2K~\cite{timofte2017ntire} and Flickr2K~\cite{lim2017enhanced} datasets. 
}

\begin{table*}[t]
	\footnotesize
	\centering
	\caption{Adaptation performance of super-resolution models on images with different degradation methods for $8{\times}$ SR.}
	\resizebox{0.95\textwidth}{!}
	{
		\begin{tabular}{c|c|ccccc}
			\toprule
			\multicolumn{1}{c|}{\multirow{2}[0]{*}{Algorithms}} & \multirow{2}[0]{*}{Degradation} & \multicolumn{1}{c}{Set5} & \multicolumn{1}{c}{Set14} & \multicolumn{1}{c}{BSDS100} & \multicolumn{1}{c}{Urban100} & \multicolumn{1}{c}{Manga109} \\
			&       & PSNR / SSIM & PSNR / SSIM & PSNR / SSIM & PSNR / SSIM & PSNR / SSIM \\
			\hline
			Nearest    & \multirow{6}[0]{*}{Nearest} & 21.22 / 0.560   &  20.11 / 0.485 &  20.64 / 0.471  & 17.76 / 0.454  & 18.51 / 0.594 \\
			EDSR~\cite{lim2017enhanced}  &       & 19.56 / 0.580 & 18.24 / 0.498 & 18.53 / 0.479 & 15.68 / 0.435 & 17.22 / {0.598} \\
			DBPN~\cite{DBPN2018}  &       & 18.80 / 0.541 & 17.36 / 0.461 & 17.94 / 0.456 & 15.07 / 0.400 & 16.67 / 0.550 \\
			RCAN~\cite{zhang2018image}  &       & 18.33 / 0.534 & 17.11 / 0.436 & 17.67 / 0.444 & 14.73 / 0.380 & 16.25 / 0.525 \\
			CinCGAN~\cite{yuan2018unsupervised} &       & 21.76 / 0.648 & 20.64 / 0.552 & 20.89 / 0.528 & 18.21 / 0.505 & 18.86 / \textbf{0.638} \\
			DRN-Adapt  &       & \textbf{23.00} / \textbf{0.715} & \textbf{21.52} / \textbf{0.561} & \textbf{21.98} / \textbf{0.539} & \textbf{19.07} / \textbf{0.518} & \textbf{19.83} / {0.613} \\
			\hline
			EDSR~\cite{lim2017enhanced}  & \multirow{5}[0]{*}{BD}  & 23.54 / 0.702 & 22.13 / 0.594 & 22.71 / 0.567 & 19.70 / 0.551 & 20.64 / 0.700 \\
			DBPN~\cite{DBPN2018}  &       & 23.05 / 0.693 & 21.65 / 0.586 & 22.50 / 0.565 & 19.28 / 0.538 & 20.16 / 0.689 \\
			RCAN~\cite{zhang2018image}  &       & 22.23 / 0.678 & 21.01 / 0.567 & 21.85 / 0.552 & 18.36 / 0.509 & 19.34 / 0.659 \\
			CinCGAN~\cite{yuan2018unsupervised}  &       & {23.39} / {0.682} & {22.14} / {0.581} & {22.73} / {0.554} & {20.36} / {0.538} & {20.29} / {0.670} \\
			DRN-Adapt  &       & \textbf{24.62} / \textbf{0.719} & \textbf{23.07} / \textbf{0.612} & \textbf{23.59} / \textbf{0.583} & \textbf{20.57} / \textbf{0.591} & \textbf{21.52} / \textbf{0.714} \\
			\bottomrule
	\end{tabular}}
	\label{tab:adaptation}%
\end{table*}%

\subsubsection{Comparison with State-of-the-art Methods}
{We compare our method with state-of-the-art SR methods in terms of both quantitative results and visual results.}
For quantitative comparison, we compare the PSNR and SSIM values of different methods for $4\times$ and $8\times$ super-resolution.
From Table~\ref{exp:2x_4x_8xsr}, our DRN-S with about 5M parameters yields {promising performance. Our DRN-L with about 10M parameters yields comparable performance with the considered methods for $4 \times$ SR and yields the best performance for $8 \times$ SR.}
{For quality comparison,}
we provide visual comparisons for our method and the considered methods (See Figure~\ref{fig:image_compare}). 
For both $4 \times$ and $8 \times$ SR, our model consistently produces sharper edges and shapes, while other baselines may give more blurry ones. The results demonstrate the effectiveness of the proposed {dual regression} scheme in generating more accurate and visually promising HR images. 
{More results are put in the supplementary.}

{
We also compare the number of parameters in different models for $4\times$ and $8\times$ SR.
}
{Due to the page limit, we only show the results for $4\times$ SR and put the $8\times$ SR in the supplementary.}
From Figure~\ref{fig:complexity_main}, our DRN-S obtains promising performance with a small number of parameters.
When we increase the number of channels and layers, the larger model DRN-L further improves the performance and obtains the best results.
Both the empirical results and the theoretical analysis in Theorem~\ref{theorem: generalization bound} show the effectiveness of the proposed dual regression scheme for image super-resolution.

\subsection{Adaptation Results on Unpaired Data}\label{sec:semi_results}

In this experiment, we apply the proposed method to a variety of real-world unpaired data.
Different from the supervised setting, 
we first consider a toy case where we evaluate SR models on the LR images with different degradation methods (\eg, Nearest and BD~\cite{zhang2017learning}).
During training, we can only access the LR images but not their corresponding HR images. 
{Then, we also apply our method to LR raw video frames from YouTube.}

\subsubsection{Datasets and Implementation Details}
{In this experiment, we obtain the paired synthetic data by downsampling existing images. Considering the real-world SR applications, all the paired data belong to a different domain from the unpaired data (See more discussions in supplementary).}
Following~\cite{russakovsky2015imagenet}, we randomly choose 3k images from ImageNet (called ImageNet3K) and obtain LR images using different degradation methods, including Nearest and BD. 
We adopt DIV2K {(Bicubic)} as the paired synthetic data\footnote{We can also use other degradation methods to obtain the paired synthetic data. We put the impact of degradation methods in supplementary.} and ImageNet3K LR images with different degradations as the unpaired data.
Note that ImageNet3K HR images are not used in our experiments.
For the SR task on video, we collect 3k raw video frames as the unpaired data to train the models.
In this section, we use our DRN-S model to evaluate the proposed adaptation algorithm and call the resultant model DRN-Adapt.
More details can be found in supplementary.

\begin{figure}[t]
	\centering
		\includegraphics[width = 0.95\columnwidth]{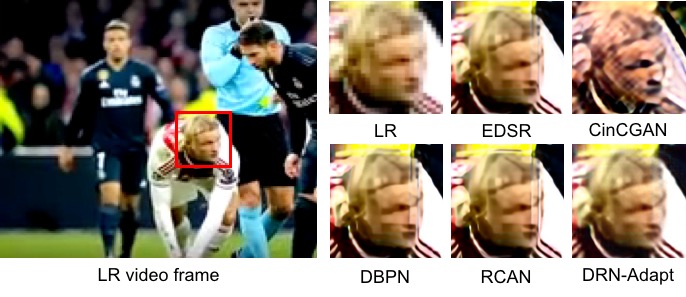}\label{fig:img8x_compare1}
	\caption{Visual comparison of model adaptation to real-world video frames (from YouTube) for $8{\times}$ SR.}
	\vspace{-10pt}
	\label{fig:video}
\end{figure}

\subsubsection{Comparison on Unpaired Synthetic Data}
To evaluate the adaptation performance on unpaired data, we compare our DRN-Adapt and the baseline methods on synthetic data. 
We report the PSRN and SSIM values of different methods for $8\times$ super-resolution in Table~\ref{tab:adaptation}.

From Table~\ref{tab:adaptation},
{DRN-Adapt} consistently outperforms the supervised methods on all the datasets.
For CycleGAN based method, CinCGAN achieves better performance than the supervised learning methods but still cannot surpass our method due to the inherent limitations mentioned before.
Note that, for Nearest LR data, we also report the recovering results of the \emph{Nearest} kernel, which is the same as the degradation method. Our method also yields a large performance improvement over this baseline.
{These results demonstrate the effectiveness of the proposed adaptation algorithm.}

\subsubsection{Comparison on Unpaired Real-world Data}
We apply our method to YouTube raw video frames, which are more challenging owing to the complicated and unknown degradation in real-world scenarios. 
{Since there are no ground-truth HR images, we only provide the visual comparison.}
From Figure~\ref{fig:video}, the generated frames from three supervised baselines (\ie, EDSR, DBPN, and RCAN) contain numerous mosaics. 
For CinCGAN, the SR results are distorted and contain a lot of noise due to the sensitivity to data differences between unpaired LR and HR images. 
By contrast, our DRN-Adapt produces visually promising images with sharper and clearer textures.
{Due to the page limit, we put more visual results in the supplementary.}

\begin{table}[t]
	\centering
	\caption{The impact of the proposed dual regression scheme on super-resolution performance in terms of PSNR score on the five benchmark datasets for 4 $\times$ SR.}
	\resizebox{0.47\textwidth}{!}{	
		\begin{tabular}{c|c|c|c|c|c|c}
			\toprule
			Model & Dual & Set5  & Set14 & BSDS100 & Urban100 & Manga109  \\
			\hline
			\multirow{2}[0]{*}{DRN-S} & \xmark & 32.53 & 28.76 & 27.68 & 26.54 & 31.21 \\
			 & $\checkmark$ & \textbf{32.68} & \textbf{28.93} & \textbf{27.78} & \textbf{26.84} & \textbf{31.52} \\
			\hline
			\multirow{2}[0]{*}{DRN-L} & \xmark & 32.61 & 28.84 & 27.72 & 26.77 & 31.39 \\
			 & $\checkmark$ & \textbf{32.74} & \textbf{28.98} & \textbf{27.83} & \textbf{27.03} & \textbf{31.73} \\
			\bottomrule
		\end{tabular}%
	}
	\label{exp:dual}%
\end{table}%

\section{Further Experiments}

\begin{table}[t]
	\centering
	\caption{Effect of the hyper-parameter $\lambda$ in Eqn.~(\ref{eq:dual_regression}) on the performance of DRN-S (testing on Set5) for 4 $\times$ SR.}
	\resizebox{0.47\textwidth}{!}
	{
		\begin{tabular}{c|ccccc}
			\toprule
			\multirow{1}[0]{*}{$\lambda$}   & 0.001 & 0.01 & 0.1 & 1.0 & 10  \\
			\hline
			PSNR on Set5  & 32.57 & 32.61 & \textbf{32.67} & 32.51 & 32.37 \\
			\bottomrule
		\end{tabular}%
	}
	\label{tab:impact_lambda}%
\end{table}%

\subsection{Ablation Study on Dual Regression Scheme} \label{sec:ablation}

We conduct an ablation study on the dual regression scheme and report the results for $4\times$ SR in Table~\ref{exp:dual}.
Compared to the baselines, the models equipped with the dual regression scheme yield better performance on all the datasets.
{These results suggest that the dual regression scheme can improve the reconstruction of HR images by introducing an additional constraint to reduce the space of the mapping function.}
We also evaluate the impact of our dual regression scheme on other models, \eg, SRResNet~\cite{ledig2016photo} based network (See more details in the supplementary).

\subsection{Effect of $\lambda$ on DRN}\label{exp:lambda}

{We conduct an experiment to investigate the impact of the hyper-parameter $\lambda$ in Eqn.~(\ref{eq:dual_regression}). From Table~\ref{tab:impact_lambda}, when we increase $\lambda$ from 0.001 to 0.1, the dual regression loss gradually becomes more important and provides powerful supervision. If we further increase $\lambda$ to 1 or 10, the dual regression loss term would overwhelm the original primal regression loss and hamper the final performance. To obtain a good tradeoff between the primal and dual regression, we set $\lambda=0.1$ in practice.
}

\begin{figure}[t]
    \centering
		\includegraphics[width = 1\columnwidth]{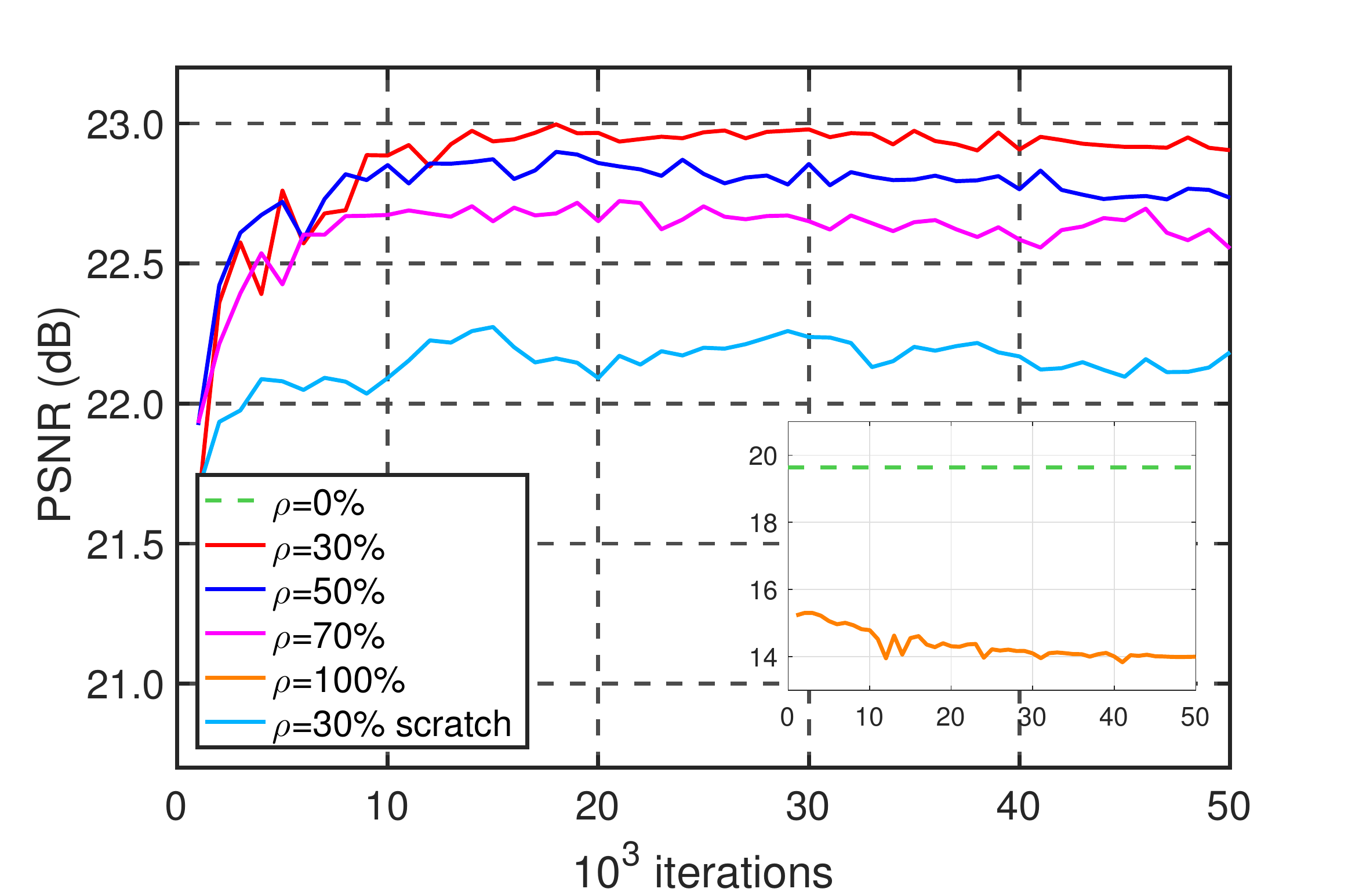}
    \caption{Comparisons of the performance on unpaired data with Nearest degradation (testing on Set5) for 4 $\times$ SR.
    }
    \label{fig:adaptation_curve}
    \vspace{-10 pt}
\end{figure}

\subsection{Effect of $\rho$ on Adaptation Algorithm} \label{exp:amount}

{We investigate the effect of $\rho$ on the proposed adaptation algorithm.}
{
We compare the performance when we change the data ratio of unpaired data $\rho$ and show the corresponding training curves in Figure~\ref{fig:adaptation_curve}.}
From Figure~\ref{fig:adaptation_curve}, when we set $\rho \in \{30\%, 50\%, 70\%\}$, the resultant models obtain better performance than the baseline model, \ie, with $\rho{=}0\%$. 
In practice, we set $\rho{=}30\%$ to obtain the best performance.
We also compare the models with and without the pretrained parameters. From Figure~\ref{fig:adaptation_curve}, the model trained from scratch yields slightly worse result but still outperforms the baseline model without adaptation. These results demonstrate the effectiveness of the proposed adaptation algorithm.

\section{Conclusion}
In this paper, we have proposed a novel dual regression scheme for paired and unpaired data.
On the paired data, we introduce an additional constraint by reconstructing LR images to reduce the space of possible functions. Thus, we can significantly improve the performance of SR models.
Furthermore, we also focus on the unpaired data and apply the dual regression scheme to real-world data., \eg, raw video frames from YouTube.
Extensive experiments on both paired and unpaired data demonstrate the superiority of our method over the considered baseline methods.

\section*{Acknowledgments}
This work was partially supported by Guangdong Provincial Scientific and Technological Funds under Grants 2018B010107001, key project of NSFC 61836003, Fundamental Research Funds for the Central Universities D2191240, Program for Guangdong Introducing Innovative and Enterpreneurial Teams 2017ZT07X183, Tencent AI Lab Rhino-Bird Focused Research Program JR201902, Guangdong Special Branch Plans Young Talent with Scientific and Technological Innovation 2016TQ03X445, Guangzhou Science and Technology Planning Project 201904010197, 
Natural Science Foundation of Guangdong Province 2016A030313437,
and Microsoft Research Asia (MSRA Collaborative Research Program). 

{
	\small
	\bibliographystyle{ieee_fullname}
	\bibliography{bib_long}

\begin{thebibliography}{10}\itemsep=-1pt

\bibitem{arbelaez2011contour}
Pablo Arbelaez, Michael Maire, Charless Fowlkes, and Jitendra Malik.
\newblock Contour detection and hierarchical image segmentation.
\newblock {\em {IEEE Transactions on Pattern Analysis and Machine
  Intelligence}}, 33(5):898--916, 2011.

\bibitem{bell2019blind}
Sefi Bell-Kligler, Assaf Shocher, and Michal Irani.
\newblock Blind super-resolution kernel estimation using an internal-gan.
\newblock In {\em Advances in Neural Information Processing Systems}, pages
  284--293, 2019.

\bibitem{DBLP:conf/bmvc/BevilacquaRGA12}
Marco Bevilacqua, Aline Roumy, Christine Guillemot, and Marie{-}Line
  Alberi{-}Morel.
\newblock Low-complexity single-image super-resolution based on nonnegative
  neighbor embedding.
\newblock In {\em BMVC}, 2012.

\bibitem{cao2018adversarial}
Jiezhang Cao, Yong Guo, Qingyao Wu, Chunhua Shen, Junzhou Huang, and Mingkui
  Tan.
\newblock Adversarial learning with local coordinate coding.
\newblock In {\em {International Conference on Machine Learning}}, 2018.

\bibitem{cao2019multi}
Jiezhang Cao, Langyuan Mo, Yifan Zhang, Kui Jia, Chunhua Shen, and Mingkui Tan.
\newblock Multi-marginal wasserstein gan.
\newblock In {\em {Advances in Neural Information Processing Systems}}, pages
  1774--1784, 2019.

\bibitem{chen2019relation}
Peihao Chen, Chuang Gan, Guangyao Shen, Wenbing Huang, Runhao Zeng, and Mingkui
  Tan.
\newblock Relation attention for temporal action localization.
\newblock {\em {IEEE Transactions on Multimedia}}, 2019.

\bibitem{chen2020intelligent}
Qi Chen, Qi Wu, Rui Tang, Yuhan Wang, Shuai Wang, and Mingkui Tan.
\newblock Intelligent home 3d: Automatic 3d-house design from linguistic
  descriptions only.
\newblock In {\em {IEEE Conference on Computer Vision and Pattern
  Recognition}}, 2020.

\bibitem{dai2019second}
Tao Dai, Jianrui Cai, Yongbing Zhang, Shu-Tao Xia, and Lei Zhang.
\newblock Second-order attention network for single image super-resolution.
\newblock In {\em {IEEE Conference on Computer Vision and Pattern
  Recognition}}, 2019.

\bibitem{guo2020multi}
Yong Guo, Jian Chen, Qing Du, Anton Van~Den Hengel, Qinfeng Shi, and Mingkui
  Tan.
\newblock Multi-way backpropagation for training compact deep neural networks.
\newblock {\em Neural networks}, 2020.

\bibitem{guo2018dual}
Yong Guo, Qi Chen, Jian Chen, Junzhou Huang, Yanwu Xu, Jiezhang Cao, Peilin
  Zhao, and Mingkui Tan.
\newblock Dual reconstruction nets for image super-resolution with gradient
  sensitive loss.
\newblock {\em arXiv preprint arXiv:1809.07099}, 2018.

\bibitem{guo2019auto}
Yong Guo, Qi Chen, Jian Chen, Qingyao Wu, Qinfeng Shi, and Mingkui Tan.
\newblock Auto-embedding generative adversarial networks for high resolution
  image synthesis.
\newblock {\em {IEEE Transactions on Multimedia}}, 2019.

\bibitem{guo2020hierarchical}
Yong Guo, Yongsheng Luo, Zhenhao He, Jin Huang, and Jian Chen.
\newblock Hierarchical neural architecture search for single image
  super-resolution.
\newblock {\em arXiv preprint arXiv:2003.04619}, 2020.

\bibitem{guo2016shallow}
Yong Guo, Mingkui Tan, Qingyao Wu, Jian Chen, Anton Van~Den Hengel, and Qinfeng
  Shi.
\newblock The shallow end: Empowering shallower deep-convolutional networks
  through auxiliary outputs.
\newblock {\em arXiv preprint arXiv:1611.01773}, 2016.

\bibitem{guo2018double}
Yong Guo, Qingyao Wu, Chaorui Deng, Jian Chen, and Mingkui Tan.
\newblock Double forward propagation for memorized batch normalization.
\newblock In {\em AAAI}, 2018.

\bibitem{guo2019nat}
Yong Guo, Yin Zheng, Mingkui Tan, Qi Chen, Jian Chen, Peilin Zhao, and Junzhou
  Huang.
\newblock Nat: Neural architecture transformer for accurate and compact
  architectures.
\newblock In {\em {Advances in Neural Information Processing Systems}}, pages
  735--747, 2019.

\bibitem{DBPN2018}
Muhammad Haris, Greg Shakhnarovich, and Norimichi Ukita.
\newblock Deep back-projection networks for super-resolution.
\newblock In {\em {IEEE Conference on Computer Vision and Pattern
  Recognition}}, 2018.

\bibitem{he2016dual}
Di He, Yingce Xia, Tao Qin, Liwei Wang, Nenghai Yu, Tieyan Liu, and Wei-Ying
  Ma.
\newblock Dual learning for machine translation.
\newblock In {\em {Advances in Neural Information Processing Systems}}, pages
  820--828, 2016.

\bibitem{he2016deep}
Kaiming He, Xiangyu Zhang, Shaoqing Ren, and Jian Sun.
\newblock Deep residual learning for image recognition.
\newblock In {\em {IEEE Conference on Computer Vision and Pattern
  Recognition}}, pages 770--778, 2016.

\bibitem{hou1978cubic}
Hsieh Hou and H Andrews.
\newblock Cubic splines for image interpolation and digital filtering.
\newblock {\em IEEE Transactions on Acoustics, Speech, and Signal Processing},
  26(6):508--517, 1978.

\bibitem{hu2019multi}
Zhibin Hu, Yongsheng Luo, Jiong Lin, Yan Yan, and Jian Chen.
\newblock Multi-level visual-semantic alignments with relation-wise dual
  attention network for image and text matching.
\newblock In {\em IJCAI}, 2019.

\bibitem{huang2015single}
Jia-Bin Huang, Abhishek Singh, and Narendra Ahuja.
\newblock Single image super-resolution from transformed self-exemplars.
\newblock In {\em {IEEE Conference on Computer Vision and Pattern
  Recognition}}, pages 5197--5206, 2015.

\bibitem{iqbal2019super}
Zohaib Iqbal, Dan Nguyen, Gilbert Hangel, Stanislav Motyka, Wolfgang Bogner,
  and Steve Jiang.
\newblock Super-resolution 1h magnetic resonance spectroscopic imaging
  utilizing deep learning.
\newblock {\em Frontiers in oncology}, 9, 2019.

\bibitem{lai2017deep}
Wei-Sheng Lai, Jia-Bin Huang, Narendra Ahuja, and Ming-Hsuan Yang.
\newblock Deep laplacian pyramid networks for fast and accurate
  super-resolution.
\newblock In {\em {IEEE Conference on Computer Vision and Pattern
  Recognition}}, 2017.

\bibitem{ledig2016photo}
Christian Ledig, Lucas Theis, Ferenc Husz{\'a}r, Jose Caballero, Andrew
  Cunningham, Alejandro Acosta, Andrew Aitken, Alykhan Tejani, Johannes Totz,
  Zehan Wang, et~al.
\newblock Photo-realistic single image super-resolution using a generative
  adversarial network.
\newblock In {\em {IEEE Conference on Computer Vision and Pattern
  Recognition}}, 2017.

\bibitem{li2019feedback}
Zhen Li, Jinglei Yang, Zheng Liu, Xiaomin Yang, Gwanggil Jeon, and Wei Wu.
\newblock Feedback network for image super-resolution.
\newblock In {\em {IEEE Conference on Computer Vision and Pattern
  Recognition}}, pages 3867--3876, 2019.

\bibitem{lim2017enhanced}
Bee Lim, Sanghyun Son, Heewon Kim, Seungjun Nah, and Kyoung~Mu Lee.
\newblock Enhanced deep residual networks for single image super-resolution.
\newblock In {\em {IEEE Conference on Computer Vision and Pattern Recognition
  Workshops}}, page~3, 2017.

\bibitem{liu2020discrimination}
Jing Liu, Bohan Zhuang, Zhuangwei Zhuang, Yong Guo, Junzhou Huang, Jinhui Zhu,
  and Mingkui Tan.
\newblock Discrimination-aware network pruning for deep model compression.
\newblock {\em arXiv preprint arXiv:2001.01050}, 2020.

\bibitem{maas2013rectifier}
Andrew~L Maas, Awni~Y Hannun, and Andrew~Y Ng.
\newblock Rectifier nonlinearities improve neural network acoustic models.
\newblock In {\em {International Conference on Machine Learning}}, volume~30,
  page~3, 2013.

\bibitem{matsui2017sketch}
Yusuke Matsui, Kota Ito, Yuji Aramaki, Azuma Fujimoto, Toru Ogawa, Toshihiko
  Yamasaki, and Kiyoharu Aizawa.
\newblock Sketch-based manga retrieval using manga109 dataset.
\newblock {\em Multimedia Tools and Applications}, 76(20), 2017.

\bibitem{mohri2012foundations}
Mehryar Mohri, Afshin Rostamizadeh, and Ameet Talwalkar.
\newblock {\em Foundations of machine learning}.
\newblock MIT Press, 2012.

\bibitem{ronneberger2015u}
Olaf Ronneberger, Philipp Fischer, and Thomas Brox.
\newblock U-net: Convolutional networks for biomedical image segmentation.
\newblock In {\em MICCAI}, pages 234--241. Springer, 2015.

\bibitem{russakovsky2015imagenet}
Olga Russakovsky, Jia Deng, Hao Su, Jonathan Krause, Sanjeev Satheesh, Sean Ma,
  Zhiheng Huang, Andrej Karpathy, Aditya Khosla, Michael Bernstein, et~al.
\newblock Imagenet large scale visual recognition challenge.
\newblock {\em {International Journal of Computer Vision}}, 115(3), 2015.

\bibitem{shi2016real}
Wenzhe Shi, Jose Caballero, Ferenc Husz{\'a}r, Johannes Totz, Andrew~P Aitken,
  Rob Bishop, Daniel Rueckert, and Zehan Wang.
\newblock Real-time single image and video super-resolution using an efficient
  sub-pixel convolutional neural network.
\newblock In {\em {IEEE Conference on Computer Vision and Pattern
  Recognition}}, pages 1874--1883, 2016.

\bibitem{timofte2017ntire}
Radu Timofte, Eirikur Agustsson, Luc Van~Gool, Ming-Hsuan Yang, Lei Zhang, Bee
  Lim, Sanghyun Son, Heewon Kim, Seungjun Nah, Kyoung~Mu Lee, et~al.
\newblock Ntire 2017 challenge on single image super-resolution: Methods and
  results.
\newblock In {\em {IEEE Conference on Computer Vision and Pattern Recognition
  Workshops}}, pages 1110--1121. IEEE, 2017.

\bibitem{tong2017image}
Tong Tong, Gen Li, Xiejie Liu, and Qinquan Gao.
\newblock Image super-resolution using dense skip connections.
\newblock In {\em {IEEE International Conference on Computer Vision}}, pages
  4809--4817, 2017.

\bibitem{ulyanov2018deep}
Dmitry Ulyanov, Andrea Vedaldi, and Victor Lempitsky.
\newblock Deep image prior.
\newblock In {\em {IEEE Conference on Computer Vision and Pattern
  Recognition}}, pages 9446--9454, 2018.

\bibitem{wang2018esrgan}
Xintao Wang, Ke Yu, Shixiang Wu, Jinjin Gu, Yihao Liu, Chao Dong, Yu Qiao, and
  Chen Change~Loy.
\newblock Esrgan: Enhanced super-resolution generative adversarial networks.
\newblock In {\em {European Conference on Computer Vision Workshops}}, pages
  0--0, 2018.

\bibitem{wang2004image}
Zhou Wang, Alan~C Bovik, Hamid~R Sheikh, and Eero~P Simoncelli.
\newblock Image quality assessment: from error visibility to structural
  similarity.
\newblock {\em {IEEE Transactions on Image Processing}}, 13(4):600--612, 2004.

\bibitem{wang2015deep}
Zhaowen Wang, Ding Liu, Jianchao Yang, Wei Han, and Thomas Huang.
\newblock Deep networks for image super-resolution with sparse prior.
\newblock In {\em {IEEE International Conference on Computer Vision}}, pages
  370--378, 2015.

\bibitem{pmlr-v70-xia17a}
Yingce Xia, Tao Qin, Wei Chen, Jiang Bian, Nenghai Yu, and Tie-Yan Liu.
\newblock Dual supervised learning.
\newblock In {\em {International Conference on Machine Learning}}, 2017.

\bibitem{xia2018model}
Yingce Xia, Xu Tan, Fei Tian, Tao Qin, Nenghai Yu, and Tie-Yan Liu.
\newblock Model-level dual learning.
\newblock In {\em {International Conference on Machine Learning}}, 2018.

\bibitem{yi2017dualgan}
Zili Yi, Hao Zhang, Ping~Tan Gong, et~al.
\newblock Dualgan: Unsupervised dual learning for image-to-image translation.
\newblock In {\em {IEEE International Conference on Computer Vision}}, 2017.

\bibitem{yuan2018unsupervised}
Yuan Yuan, Siyuan Liu, Jiawei Zhang, Yongbing Zhang, Chao Dong, and Liang Lin.
\newblock Unsupervised image super-resolution using cycle-in-cycle generative
  adversarial networks.
\newblock In {\em {IEEE Conference on Computer Vision and Pattern Recognition
  Workshops}}, pages 701--710, 2018.

\bibitem{zeng2019breaking}
Runhao Zeng, Chuang Gan, Peihao Chen, Wenbing Huang, Qingyao Wu, and Mingkui
  Tan.
\newblock Breaking winner-takes-all: Iterative-winners-out networks for weakly
  supervised temporal action localization.
\newblock {\em {IEEE Transactions on Image Processing}}, 28(12), 2019.

\bibitem{zeng2019graph}
Runhao Zeng, Wenbing Huang, Mingkui Tan, Yu Rong, Peilin Zhao, Junzhou Huang,
  and Chuang Gan.
\newblock Graph convolutional networks for temporal action localization.
\newblock In {\em {IEEE International Conference on Computer Vision}}, Oct
  2019.

\bibitem{zeng2020dense}
Runhao Zeng, Haoming Xu, Wenbing Huang, Peihao Chen, Mingkui Tan, and Chuang
  Gan.
\newblock Dense regression network for video grounding.
\newblock In {\em {IEEE Conference on Computer Vision and Pattern
  Recognition}}, June 2020.

\bibitem{zeyde2010single}
Roman Zeyde, Michael Elad, and Matan Protter.
\newblock On single image scale-up using sparse-representations.
\newblock In {\em International Conference on Curves and Surfaces}, pages
  711--730. Springer, 2010.

\bibitem{zhang2017learning}
Kai Zhang, Wangmeng Zuo, and Lei Zhang.
\newblock Learning a single convolutional super-resolution network for multiple
  degradations.
\newblock In {\em {IEEE Conference on Computer Vision and Pattern
  Recognition}}, 2018.

\bibitem{zhang2019ranksrgan}
Wenlong Zhang, Yihao Liu, Chao Dong, and Yu Qiao.
\newblock Ranksrgan: Generative adversarial networks with ranker for image
  super-resolution.
\newblock In {\em Proceedings of the IEEE International Conference on Computer
  Vision}, pages 3096--3105, 2019.

\bibitem{zhang2019whole}
Yifan Zhang, Hanbo Chen, Ying Wei, Peilin Zhao, Jiezhang Cao, et~al.
\newblock From whole slide imaging to microscopy: Deep microscopy adaptation
  network for histopathology cancer image classification.
\newblock In {\em MICCAI}. Springer, 2019.

\bibitem{zhang2018image}
Yulun Zhang, Kunpeng Li, Kai Li, Lichen Wang, Bineng Zhong, and Yun Fu.
\newblock Image super-resolution using very deep residual channel attention
  networks.
\newblock In {\em {European Conference on Computer Vision}}, 2018.

\bibitem{zhang2019collaborative}
Yifan Zhang, Ying Wei, Peilin Zhao, Shuaicheng Niu, et~al.
\newblock Collaborative unsupervised domain adaptation for medical image
  diagnosis.
\newblock In {\em Medical Imaging meets NeurIPS}, 2019.

\bibitem{zhang2018deep}
Ying Zhang, Tao Xiang, Timothy~M Hospedales, and Huchuan Lu.
\newblock Deep mutual learning.
\newblock In {\em {IEEE Conference on Computer Vision and Pattern
  Recognition}}, 2018.

\bibitem{zhao2018unsupervised}
Tianyu Zhao, Changqing Zhang, Wenqi Ren, Dongwei Ren, and Qinghua Hu.
\newblock Unsupervised degradation learning for single image super-resolution.
\newblock {\em arXiv preprint arXiv:1812.04240}, 2018.

\bibitem{zhou2019kernel}
Ruofan Zhou and Sabine Susstrunk.
\newblock Kernel modeling super-resolution on real low-resolution images.
\newblock In {\em {IEEE International Conference on Computer Vision}}, 2019.

\bibitem{zhu2017unpaired}
Jun-Yan Zhu, Taesung Park, Phillip Isola, and Alexei~A Efros.
\newblock Unpaired image-to-image translation using cycle-consistent
  adversarial networks.
\newblock In {\em {IEEE International Conference on Computer Vision}}, 2017.

\end{thebibliography}
}

% \eat
{

\clearpage
\onecolumn
\appendix

\begin{center}
	{
		\Large{\textbf{Supplementary Materials for ``\mytitle''}}
	}
\end{center}

We organize our supplementary materials as follows. First, we provide the derivation of generalization error bound for the dual regression scheme in Section~\ref{sec:proof}. Second, we provide more details on the architecture of the proposed DRN model in Section~\ref{sec:Model_details}. Third, we provide more implementation details on the training method for the SR tasks with paired data and unpaired data in Section~\ref{sec:implementation_details_sup}. Fourth, we conduct more ablation studies on the proposed dual regression scheme in Section~\ref{sec:ablatioon_sup}. Last, we report more visual comparison results in Section~\ref{sec:more_results_sup}.

\section{Theoretical Analysis}\label{sec:proof}

In this section, we will analyze the generalization bound for the proposed method.
The generalization error of the dual learning scheme is to measure how accurately the algorithm predicts for the unseen test data in the primal and dual tasks. Firstly, we will introduce the definition of the generalization error as follows:
\begin{deftn}
	Given an underlying distribution $ \mS $ and hypotheses $ P \in \mP $ and $ D \in \mD $ for the primal and dual tasks, where $ \mP = \{ P_{\theta_{\bx\by}}(\bx); \theta_{\bx\by} \in \Theta_{\bx\by} \} $ and $ \mD = \{ D_{\theta_{\by\bx}}(\by); \theta_{\by\bx} \in \Theta_{\by\bx} \}$, and $ \Theta_{\bx\by} $ and $ \Theta_{\by\bx} $ are parameter spaces, respectively, the generalization error (expected loss)
	% 		of h 
	is defined by:
	\begin{align*}
	E(P, D) = \mmE_{(\bx, \by) \sim \mP} \left[ \mL_P (P(\bx), \by) + {\lambda} \mL_D (D(P(\bx)), \bx) \right], \; \forall P \in \mP, D \in \mD.
	\end{align*}
\end{deftn}

In practice, the goal of the dual learning is to optimize the bi-directional tasks. For any $ P \in \mP $ and $ D \in \mD $, we define the empirical loss on the $ N $ samples as follows:
\begin{align}
\hat{E}(P, D) = \frac{1}{N} \sum_{i=1}^{N} \mL_P (P(\bx_i), \by_i) + {\lambda} \mL_D (D(P(\bx_i)), \bx_i)
\end{align}
Following \cite{mohri2012foundations}, we define Rademacher complexity for dual learning in this paper.
We define the function space as $ \mH_{dual} \in \mP \times \mD $, this Rademacher complexity can measure the complexity of the function space, that is it can capture the richness of a family of the primal and the dual models.
For our application, we mildly rewrite the definition of Rademacher complexity in  \cite{mohri2012foundations} as follows:
\begin{deftn} \textbf{\emph{(Rademacher complexity of dual learning) }} \label{Def: Rademacher Complexity}
	Given an underlying distribution $ \mS $, and its empirical distribution $ \mZ = \{\bz_1, \bz_2, \cdots,\bz_N\} $, where $ \bz_i = (\bx_i, \by_i) $, then the Rademacher complexity of dual learning is defined as:
	\begin{align*}
	R_N^{DL} (\mH_{dual}) = \mmE_{\mZ} \left[ \hat{R}_{\mZ} (P, D) \right], \; \forall P \in \mP, D \in \mD,
	\end{align*}
	where $ \hat{R}_{\mZ} (P, D) $ is its empirical Rademacher complexity defined as:
	\begin{align*}
	\hat{R}_{\mZ} (P, D) = \mmE_{\sigma} \left[ \sup_{(P, D) \in \mH_{dual}}  \frac{1}{N} \sum_{i=1}^{N} \sigma_i (\mL_P (P(\bx_i), \by_i) + {\lambda} \mL_D (D(P(\bx_i)), \bx_i))  \right].
	\end{align*}
	where $ \sigma = \{ \sigma_1, \sigma_2, \cdots, \sigma_N \} $ are independent uniform $ \{\pm 1\} $-valued random variables with $ p(\sigma_i = 1) = p(\sigma_i = -1) = \frac{1}{2} $.
\end{deftn}

\paragraph{Generalization bound.}
Here, we analyze the generalization bound for the proposed dual regression scheme. We first start with a simple case of finite function space. Then, we generalize it to a more general case with infinite function space.
\begin{thm}
	Let $ \mL_P (P(\bx), \by) + {\lambda} \mL_D (D(P(\bx)), \bx) $ be a mapping from $ \mX \times \mY $ to $ [0, C] $, and suppose the function space $ \mH_{dual} $ is finite, then for any $ \delta > 0 $, with probability at least $ 1-\delta $, the following inequality holds for all $ (P, D) \in \mH_{dual} $:
	\begin{align*}
	E(P, D) \leq \hat{E}(P, D) + C \sqrt{\frac{\log|\mH_{dual}| + \log \frac{1}{\delta}}{2N}}.
	\end{align*}
\end{thm}

\begin{proof}
	Based on Hoeffding's inequality, since $ \mL_P (P(\bx), \by) + {\lambda}  \mL_D (D(P(\bx)), \bx) $ is bounded in $ [0, C] $, for any $ (P, D) \in \mH_{dual} $, then
	\begin{align*}
	P\left[ E(P, D) - \hat{E}(P, D) > \epsilon \right] \leq e^{-\frac{2N\epsilon^2}{C^2}}
	\end{align*}
	Based on the union bound, we have
	\begin{align*}
	&P\left[ \exists (P, D) \in \mH_{dual}: E(P, D) - \hat{E}(P, D) > \epsilon \right] \\
	\leq& \sum_{(P, D) \in \mH_{dual}} P \left[ E(P, D) - \hat{E}(P, D) > \epsilon \right] \\
	\leq& |\mH_{dual}| e^{-\frac{2N\epsilon^2}{C^2}}.
	\end{align*}
	Let $ |\mH_{dual}| e^{-\frac{2N\epsilon^2}{C^2}} = \delta $, we have $ \epsilon = C \sqrt{\frac{\log|\mH_{dual}| + \log \frac{1}{\delta}}{2N}} $ and conclude the theorem.
\end{proof}
This theorem shows that a larger sample size $ N $ and smaller function space can guarantee the generalization.
Next, we will give a generalization bound of a general case of an infinite function space using Rademacher complexity.	

\begin{thm} \label{theorem: generalization_bound_sup}
	Let $ \mL_P(P(\bx), \by) {+} {\lambda}\mL_D (D(P(\bx)), \bx) $ be a mapping from $ \mX {\times} \mY $ to $ [0, C] $ with the upper bound $C$, and the function space $ \mH_{dual} $ be infinite. Then, for any $ \delta {>} 0 $, with probability at least $ 1{-}\delta $, the generalization error $E(P, D)$ (\ie, expected loss)  satisfies for all $ (P, D) {\in} \mH_{dual} $:
	\begin{align}\label{ineqn:E_PD_bound}
	E(P, D) \leq \hat{E}(P, D) {+} 2 \hat{R}_{\mZ}^{DL}(\mH_{dual}) {+} 3C \sqrt{\frac{1}{2N} \log\left(\frac{1}{\delta}\right)},
	\end{align}
	where $N$ is the number of samples and $\hat{R}_{\mZ}^{DL}$ is the empirical Rademacher complexity of dual learning. 
	Let $\mB(P, D)$ be the generalization bound of the dual regression SR, \ie $\mB(P, D){=}2\hat{R}_{\mZ}^{DL}(\mH_{dual}) {+} 3C \sqrt{\frac{1}{2N} \log\left(\frac{1}{\delta}\right)}$, we have
	\begin{align}\label{ineqn:B_PD_bound}
	\mB(P, D) \leq \mB(P),
	\end{align}
	where $\mB(P), P{\in}\mH$ is the generalization bound of standard supervised learning \textit{w.r.t.} the Rademacher complexity $\hat{R}_{\mZ}^{SL}(\mH)$.
\end{thm}
\begin{proof}
	Based on Theorem 3.1 in \cite{mohri2012foundations}, we extend a case for $ \mL_P (P(\bx), \by) + {\lambda}  \mL_D (D(P(\bx)), \bx) $ bounded in $ [0, C] $, and we have the generalization bound in (\ref{ineqn:E_PD_bound}). 
	According to the definition of Rademacher complexity, we have $\hat{R}_{\mZ}^{DL}(\mH_{dual}) {\leq} \hat{R}_{\mZ}^{SL}(\mH) $ because the capacity of the function space $ \mH_{dual} \small{\in} \mP \small{\times} \mD $ is smaller than the capacity of the function space $ \mH \small{\in} \mP $.
	With the same number of samples, we have $\mB(P, D) {\leq} \mB(P)$.
\end{proof}
Theorem \ref{theorem: generalization_bound_sup} shows that with probability at least $ 1-\delta $, the generalization error is smaller than $ 2 R_N^{DL} + C \sqrt{\frac{1}{2N} \log(\frac{1}{\delta})} $ or $ 2 \hat{R}_{\mZ}^{DL} + 3C \sqrt{\frac{1}{2N} \log(\frac{1}{\delta})} $.
It suggests that using the function space with larger capacity and more samples can guarantee better generalization.
Moreover, the generalization bound of dual learning is more general for the case that the loss function $ \mL_P (P(\bx), \by) + {\lambda} \mL_D (D(P(\bx)), \bx) $ is bounded by $ [0, C] $, which is different from \cite{pmlr-v70-xia17a}.

\begin{remark}
	Based on the definition of Rademacher complexity, the capacity of the function space $ \mH_{dual} \small{\in} \mP \small{\times} \mD $ is smaller than the capacity of the function space $ \mH \small{\in} \mP $ or $ \mH \small{\in} \mD $ in traditional supervised learning, \ie, $ \hat{R}_{\mZ}^{DL} \leq \hat{R}_{\mZ}^{SL} $, where $ \hat{R}_{\mZ}^{SL} $ is Rademacher complexity defined in supervised learning. In other words, dual learning has a smaller generalization bound than supervised learning and the proposed dual regression model helps the primal model to achieve more accurate SR predictions.
\end{remark}

\section{Model Details of Dual Regression Network} \label{sec:Model_details}
Deep neural networks (DNNs) have achieved great success in image classification~\cite{guo2018double,guo2020multi,guo2016shallow,guo2019nat}, image generation~\cite{guo2019auto,cao2018adversarial}, and image restoration~\cite{guo2018dual,guo2020hierarchical}. In this paper, we propose a novel Dual Regression Network (DRN), which contains a primal model and a dual model.
{Specifically, the primal model contains 2 basic blocks for $4\times$ SR and 3 blocks for $8\times$ SR. {To form a closed-loop, according to the architecture design of the primal model}, there are 2 dual models for $4\times$ SR and 3 dual models for $8\times$ SR, respectively.}

{
	Let $B$ be the number of RCABs~\cite{zhang2018image} and $F$ be the number of base feature channels.
	For $4\times$ SR, we set $B=30$ and $F=16$ for DRN-S and $B=40$ and $F=20$ for DRN-L. 
	For $8\times$ SR, we set $B=30$ and $F=8$ for DRN-S and $B=36$ and $F=10$ for DRN-L. 
	Moreover, we set the reduction ratio $r=16$ in all RCABs for our DRN model and set the negative slope to 0.2 for all LeakyReLU in DRN.}
We show the detailed architecture of the $8\times$ DRN model in Table~\ref{tab:model_details}. {To obtain the $4\times$ model, one can simply remove one basic block from the $8\times$ model.}

{
	As shown in Table~\ref{tab:model_details}, we use Conv(1,1) and Conv(3,3) to represent the convolution layer with the kernel size of $1 \times 1$ and $3 \times 3$, respectively. We use Conv$_{s2}$ to represent the convolution layer with the stride of $2$.
	Following the settings of EDSR~\cite{lim2017enhanced}, we build the Upsampler with one convolution layer and one pixel-shuffle~\cite{shi2016real} layer to upscale the feature maps.}
Moreover, we use h and w to represent the height and width of the input LR images. Thus, the shape of output images should be 8h $\times$ 8w for the $8\times$ model. 

\begin{table*}[h]
	\renewcommand\thetable{A}
	\centering
	\caption{Detailed model design of the proposed {$8\times$} {DRN}.}
	% 	\resizebox{0.9\textwidth}{!}
	{
		\begin{tabular}{c|c|c|c}
			\toprule
			Module &  Module details & Input shape    & Output shape \\
			\hline
			Head  &  Conv(3,3)  & (3, 8h, 8w) & (1$F$, 8h, 8w) \\
			\hline
			Down 1  & Conv$_{s2}$-LeakyReLU-Conv & (1$F$, 8h, 8w) & (2$F$, 4h, 4w) \\
			\hline
			Down 2 & Conv$_{s2}$-LeakyReLU-Conv & (2$F$, 4h, 4w)  & (4$F$, 2h, 2w) \\
			\hline
			Down 3 & Conv$_{s2}$-LeakyReLU-Conv & (4$F$, 2h, 2w) &  (8$F$, 1h, 1w) \\
			\hline
			\multirow{3}{*}{Up 1} & $B$ RCABs &  (8$F$, 1h, 1w) & (8$F$, 1h, 1w) \\
			\cline{2-4}    \multicolumn{1}{c|}{} & $2\times$ Upsampler &  (8$F$, 1h, 1w) & (8$F$, 2h, 2w) \\
			\cline{2-4}    \multicolumn{1}{c|}{} & Conv(1,1) & (8$F$, 2h, 2w)  & (4$F$, 2h, 2w) \\
			\hline
			Concatenation 1 &  Concatenation of the output of Up 1 and Down 2 & (4$F$, 2h, 2w) $\oplus$ (4$F$, 2h, 2w) & (8$F$, 2h, 2w) \\
			\hline
			\multirow{3}{*}{Up 2} & $B$ RCABs &  (8$F$, 2h, 2w) & (8$F$, 2h, 2w) \\
			\cline{2-4}    \multicolumn{1}{c|}{} & $2\times$ Upsampler & (8$F$, 2h, 2w)  & (8$F$, 4h, 4w) \\
			\cline{2-4}    \multicolumn{1}{c|}{} & Conv(1,1) & (8$F$, 4h, 4w)  & (2$F$, 4h, 4w) \\
			\hline
			Concatenation 2  & Concatenation of the output of Up 2 and Down 1 & (2$F$, 4h, 4w) $\oplus$ (2$F$, 4h, 4w)  & (4$F$, 4h, 4w) \\
			\hline
			\multirow{3}{*}{Up 3} & $B$ RCABs  & (4$F$, 4h, 4w) & (4$F$, 4h, 4w) \\
			\cline{2-4}    \multicolumn{1}{c|}{}  & $2\times$ Upsampler & (4$F$, 4h, 4w) & (4$F$, 8h, 8w) \\
			\cline{2-4}    \multicolumn{1}{c|}{}  & Conv(1,1) & (4$F$, 8h, 8w) & (1$F$, 8h, 8w) \\
			\hline
			Concatenation 3  & Concatenation of the output of Up3 and Head & (1$F$, 8h, 8w) $\oplus$ (1$F$, 8h, 8w)  & (2$F$, 8h, 8w) \\
			\hline
			Tail 0  & Conv(3,3) & (8$F$, 1h, 1w)  & (3, 1h, 1w) \\
			\hline
			Tail 1  & Conv(3,3) & (8$F$, 2h, 2w)  & (3, 2h, 2w) \\
			\hline
			Tail 2  & Conv(3,3) & (4$F$, 4h, 4w)  & (3, 4h, 4w) \\
			\hline
			Tail 3  & Conv(3,3)  & (2$F$, 8h, 8w) & (3, 8h, 8w) \\
			\hline
			Dual 1 & Conv$_{s2}$-LeakyReLU-Conv & (3, 8h, 8w)  & (3, 4h, 4w) \\
			\hline
			Dual 2 & Conv$_{s2}$-LeakyReLU-Conv & (3, 4h, 4w)  & (3, 2h, 2w) \\
			\hline
			Dual 3 & Conv$_{s2}$-LeakyReLU-Conv & (3, 2h, 2w)  & (3, 1h, 1w) \\
			\bottomrule
		\end{tabular}%
	}
	\label{tab:model_details}%
\end{table*}%

\section{More Implementation Details} \label{sec:implementation_details_sup}

\subsection{Supervised Image Super-Resolution}
\noindent \textbf{Training data.}
{Following~\cite{wang2018esrgan}, }
we train our model on {DIV2K~\cite{timofte2017ntire} and Flickr2K~\cite{lim2017enhanced} datasets, which contain $800$ and $2650$ training images separately.
	We use the RGB input patches of size $48 \times 48$ from LR images and the corresponding HR patches as the paired training data, and augment the training data following the method in~\cite{lim2017enhanced, zhang2018image}.}

\noindent \textbf{Test data.}
{For quantitative comparison on paired data, we evaluate different SR models using five benchmark datasets, including SET5~\cite{DBLP:conf/bmvc/BevilacquaRGA12}, SET14~\cite{zeyde2010single}, BSDS100~\cite{arbelaez2011contour}, URBAN100~\cite{huang2015single} and MANGA109~\cite{matsui2017sketch}.}

\noindent \textbf{Implementation details.}
{For training, we apply Adam with $\beta_1 = 0.9$, $\beta_2 = 0.99$ and set minibatch size as 32.
	The learning rate is initialized to $10^{-4}$ and decreased to $10^{-7}$ with a cosine annealing out of $10^6$ iterations in total.}

\subsection{Adaptation to Real-world Scenarios with Unpaired Data}
\noindent \textbf{Training data.}
{To obtain the unpaired synthetic data, we randomly choose 3k images from ImageNet~\cite{russakovsky2015imagenet} (called ImageNet3k) and obtain the LR images using different degradation methods, including Nearest and BD.}
More specifically, we use Matlab to obtain the Nearest data.
The BD data is obtained using the Gaussian kernel with size $7\times7$ and a standard deviation of $1.6$. Note that ImageNet3K HR images are not used in our experiments.
Moreover, we collect 3k LR raw video frames from YouTube as the unpaired real-world data to evaluate the proposed DRN in a more general and challenging case. 
{More critically, we use both paired data (DIV2K~\cite{timofte2017ntire}) and unpaired data to train the proposed models.}

\noindent \textbf{Test data.}
{For quantitative comparison on unpaired synthetic data, we obtain the LR images of five benchmark datasets using Nearest and BD degradation methods separately.}

\noindent \textbf{Implementation details.}
{We train a DRN-Adapt model for each kind of unpaired data, \ie, Nearest data, BD data, and video frames collected from YouTube. Thus, there are 3 DRN-adapt models in total. 
	And We also train a CinCGAN~\cite{zhu2017unpaired} model for each kind of unpaired data for comparison.
}
{
	Based on pretrained DRN-S, We train our DRN-Adapt models with a learning rate of $10^{-4}$ and the data ratio of unpaired data $\rho=30\%$ for a total of $ 10^5 $ iterations. 
	Moreover, we apply Adam with $\beta_1 = 0.9$, $\beta_2 = 0.99$ to optimize the models, and set minibatch size as 16.
} 

\section{More Ablation Studies on Dual Regression} \label{sec:ablatioon_sup}

{In this section, we first provide an additional ablation study of the dual regression scheme on other architectures. Then, we investigate the effect of the dual regression scheme on HR images. Last, we investigate the impact of different degradation methods to obtain paired synthetic data.
}

\subsection{Effect of Dual Regression Scheme on Other Architectures} \label{sec:dual_ablatioon_sup}
{To verify the impact of the dual regression scheme, we also conduct an ablation study of the dual network for SRResNet (see architecture in Figure~\ref{fig:srresnet_dual_sup}).
	``SRResNet + Dual'' denotes the baseline SRResNet equipped with the dual regression scheme. 
	From Table~\ref{exp:dual_srresnet}, the model with the dual regression scheme consistently outperforms the baseline counterpart, which {further} demonstrates the effectiveness of our method.}

\begin{table}[h]
	\renewcommand\thetable{B}
	\small
	\centering
	\caption{The impact of the proposed dual regression scheme on the SRResNet model in terms of PSNR score on the five benchmark datasets for $4\times$ {SR}.}
	% 		\resizebox{0.47\textwidth}{!}{	
	\begin{tabular}{c|c|c|c|c|c}
		\toprule
		Method & Set5  & Set14 & BSDS100 & Urban100 & Manga109  \\
		\hline
		SRResNet & 32.26 & 28.53 & 27.61 & 26.24 & 31.03 \\
		SRResNet + Dual & \textbf{32.47} & \textbf{28.77} & \textbf{27.70} & \textbf{26.58} & \textbf{31.24} \\
		\bottomrule
	\end{tabular}%
	% 		}
	\label{exp:dual_srresnet}%
\end{table}%

\begin{figure}[h]
	\renewcommand\thefigure{A}
	\centering
	\includegraphics[width = 0.9\columnwidth]{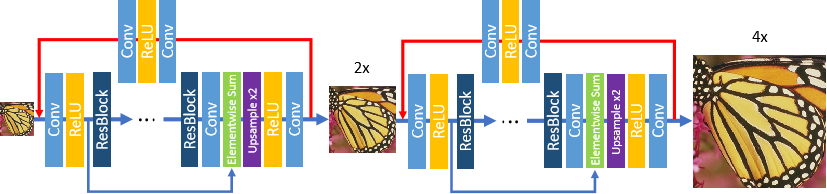}
	\caption{The SRResNet architecture equipped with the proposed dual regression scheme for $4\times$ SR.}
	\label{fig:srresnet_dual_sup}
	% \vspace{-20pt}
\end{figure}

\subsection{Effect of the Dual Regression on HR Data} \label{sec:dual_hr_sup}
{As mentioned in Section 3.1, one can also add a dual regression constraint on the HR domain, \ie, downscaling and upscaling to reconstruct the original HR images. 
	In this {experiment}, we investigate the impact of dual regression loss on HR data and show the results in Table~\ref{exp:dual_hr}.
	For convenience, we use 
	``DRN-S with dual HR'' to represent the model with the regression on both LR and HR images. From Table~\ref{exp:dual_hr}, 
	DRN-S yields comparable performance with ``DRN-S with dual HR'' while only needs half the computation cost. Thus, it is not necessary to apply the dual regression on HR images in practice. }

\begin{table}[h]
	\renewcommand\thetable{C}
	\small
	\centering
	\caption{The impact of the dual regression loss on HR data for 4$\times$ SR. We take DRN-S as the baseline model.}
	% 		\resizebox{0.47\textwidth}{!}{	
	\begin{tabular}{c|c|c|c|c|c|c}
		\toprule
		Method & MAdds & Set5  & Set14 & BSDS100 & Urban100 & Manga109  \\
		\hline
		DRN-S with dual HR & 51.20G & 32.69 & 28.93 & 27.79 & 26.85 & 31.54 \\
		DRN-S (Ours) & 25.60G & 32.68 & 28.93 & 27.78 & 26.84 & 31.52 \\
		\bottomrule
	\end{tabular}%
	% 		}
	\label{exp:dual_hr}%
\end{table}%

\subsection{Impact of Different Degradation Methods to Obtain Paired Synthetic Data}
{In this experiment, we investigate the impact of different degradation methods to obtain paired synthetic data. We change kernel from Bicubic to Nearest and evaluate the adaptation models on BD data. From Table~\ref{exp:different_downsampling}, DRN-Adapt obtain similar results when we use different degradation methods to obtain the paired synthetic data. 
}

\begin{table}[h]
	\renewcommand\thetable{F}
	\small
	\centering
	\caption{{The impact of different degradation methods on DRN-Adapt for 8$\times$ SR.}}
	\begin{tabular}{c|c|c|c|c|c}
		\toprule
		Degradation Method & Set5  & Set14 & BSDS100 & Urban100 & Manga109  \\
		\hline
		Nearest & 24.60  & 23.03  & \textbf{23.60}  & \textbf{20.61}  & 21.46  \\
		Bicubic & \textbf{24.62} & \textbf{23.07} & {23.59} & {20.57} & \textbf{21.52} \\
		\bottomrule
	\end{tabular}
	\label{exp:different_downsampling}
\end{table}

\section{More Comparisons and Results} \label{sec:more_results_sup}
{For supervised super-resolution, we put more visual results in this section shown in Figures~\ref{fig:image_compare_4x_sup} and~\ref{fig:image_compare_8x_sup}, respectively.
	Considering the scenario with unpaired data, we put more visual results on real-world unpaired data (See Figure~\ref{fig:video_sup}). }
{From these results, our models are able to produce the images with sharper edges and clearer textures than state-of-the-art methods.}

\begin{figure*}[ht]
	\renewcommand\thefigure{C}
	\centering
	\subfigure{
		\includegraphics[width = 0.73\columnwidth]{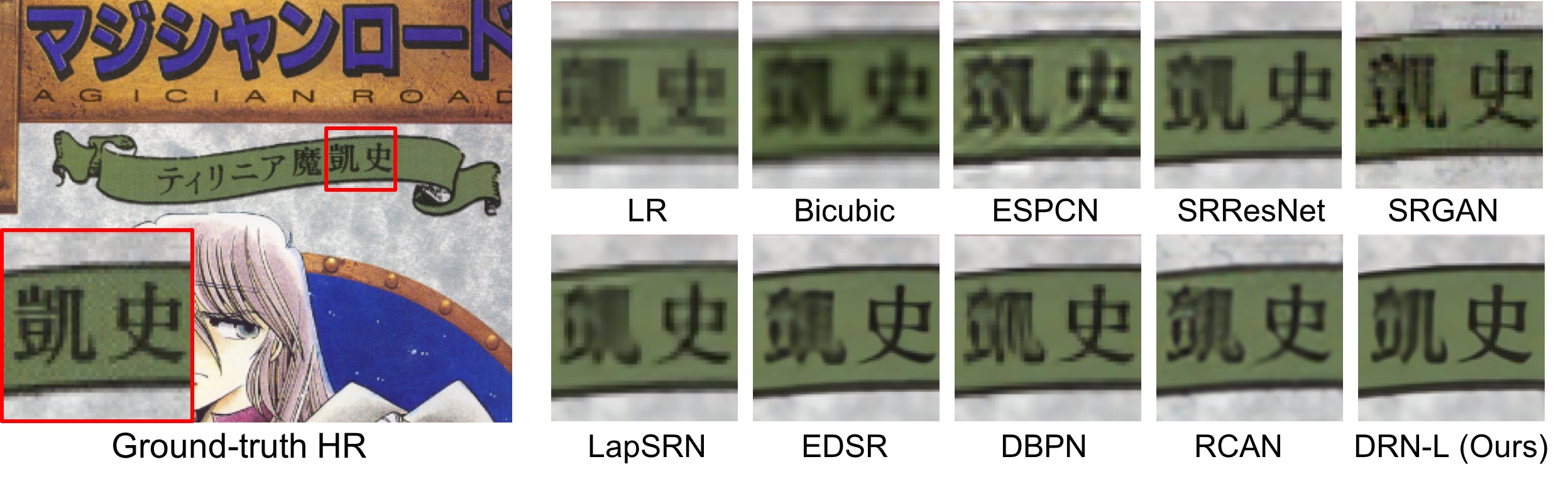}\label{fig:img4x_compare1_sup} 
	}
	\subfigure{
		\includegraphics[width = 0.73\columnwidth]{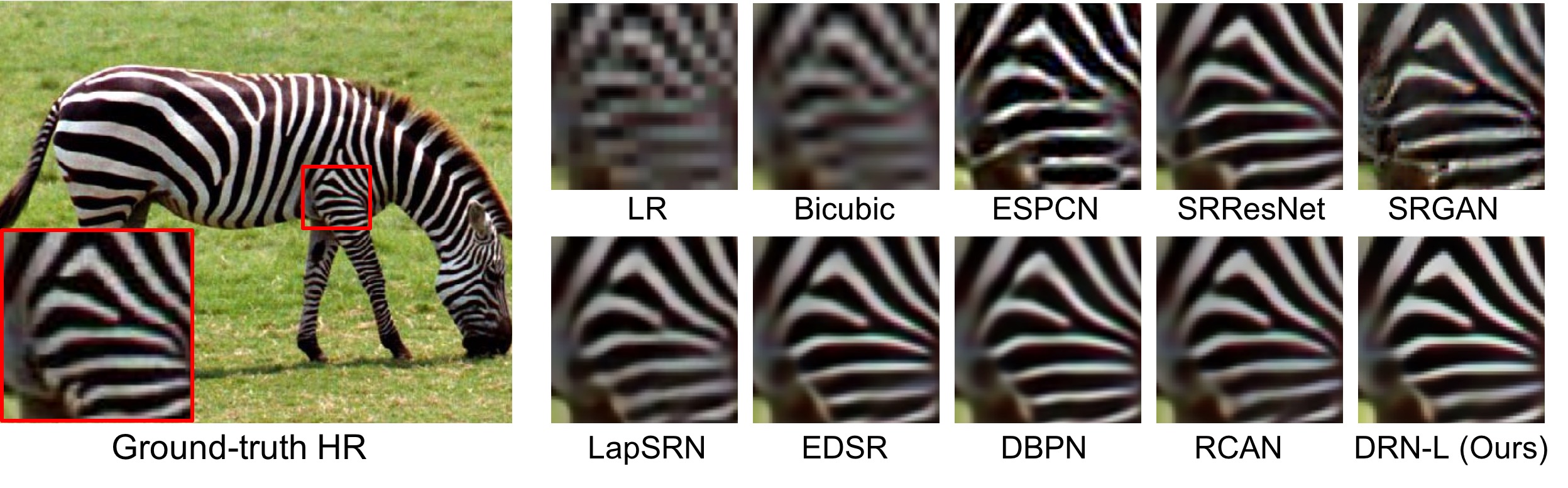}\label{fig:img4x_compare2_sup} 
	}	
	\subfigure{
		\includegraphics[width = 0.73\columnwidth]{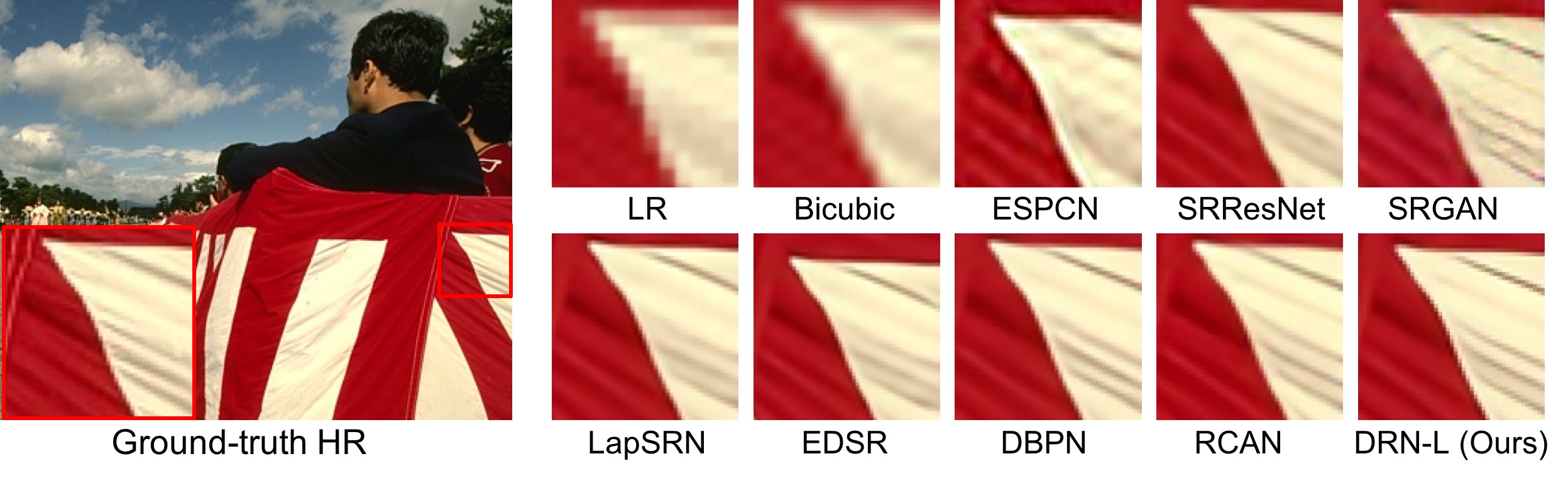}\label{fig:img4x_compare3_sup}
	}
	\caption{Visual comparison for $4\times$ image super-resolution on benchmark datasets.}
	\label{fig:image_compare_4x_sup}
\end{figure*}

\begin{figure*}[ht]
	\renewcommand\thefigure{D}
	\centering
	\subfigure{
		\includegraphics[width = 0.73\columnwidth]{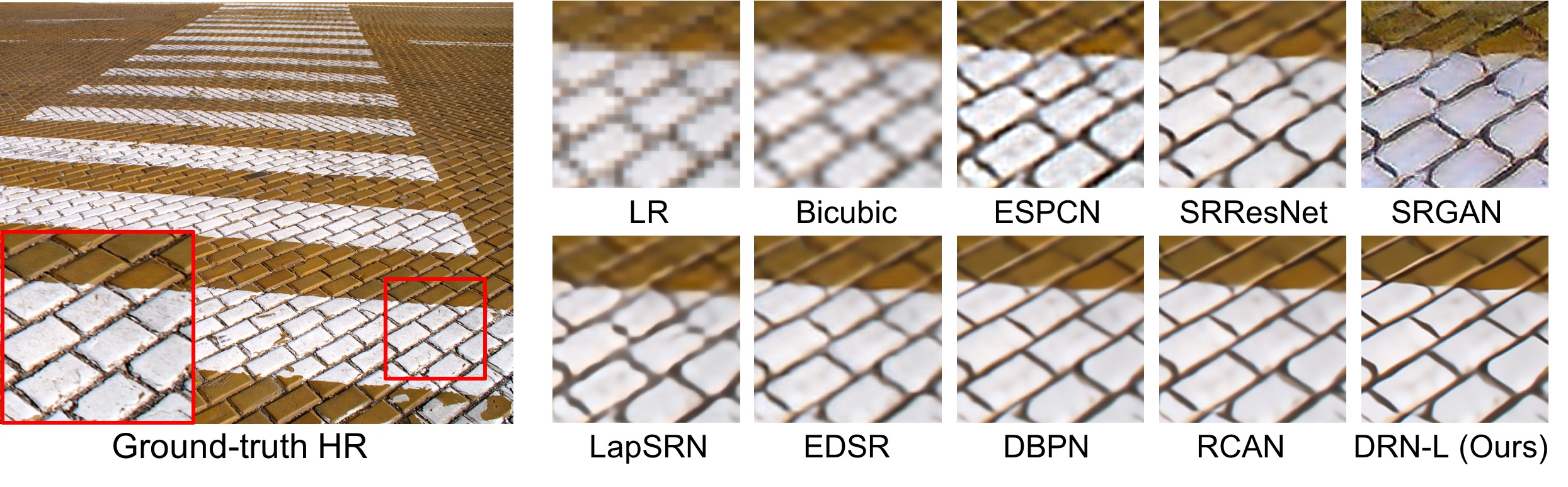}\label{fig:img8x_compare1_sup}
	}
	\subfigure{
		\includegraphics[width = 0.73\columnwidth]{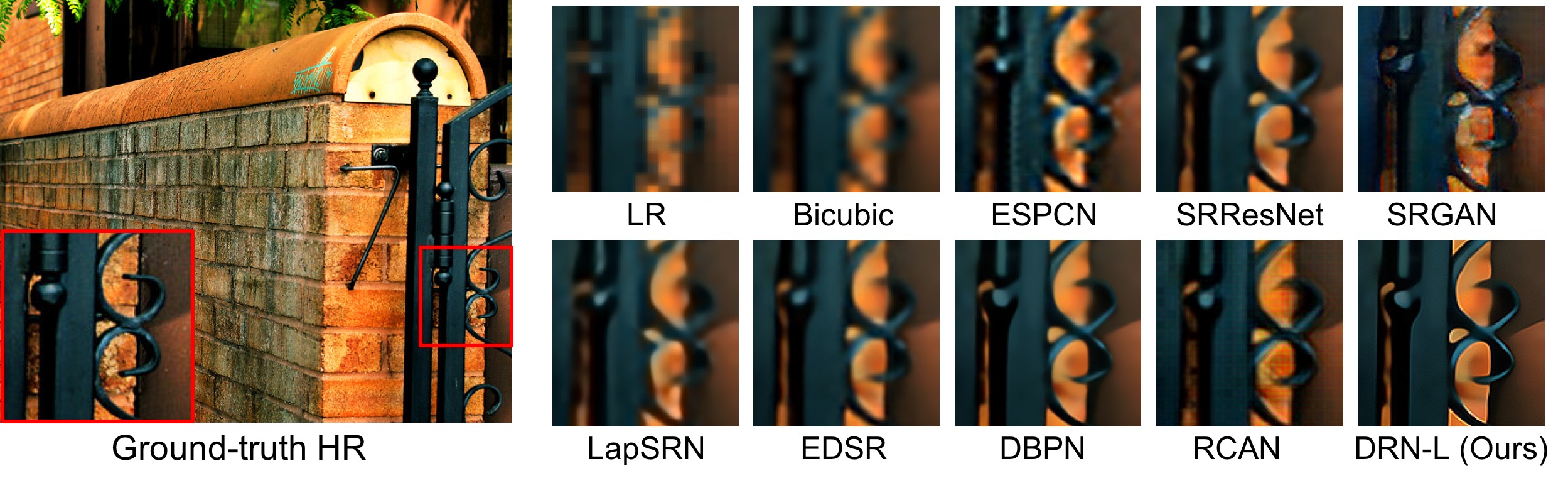}\label{fig:img8x_compare2_sup}
	}	
	\subfigure{
		\includegraphics[width = 0.73\columnwidth]{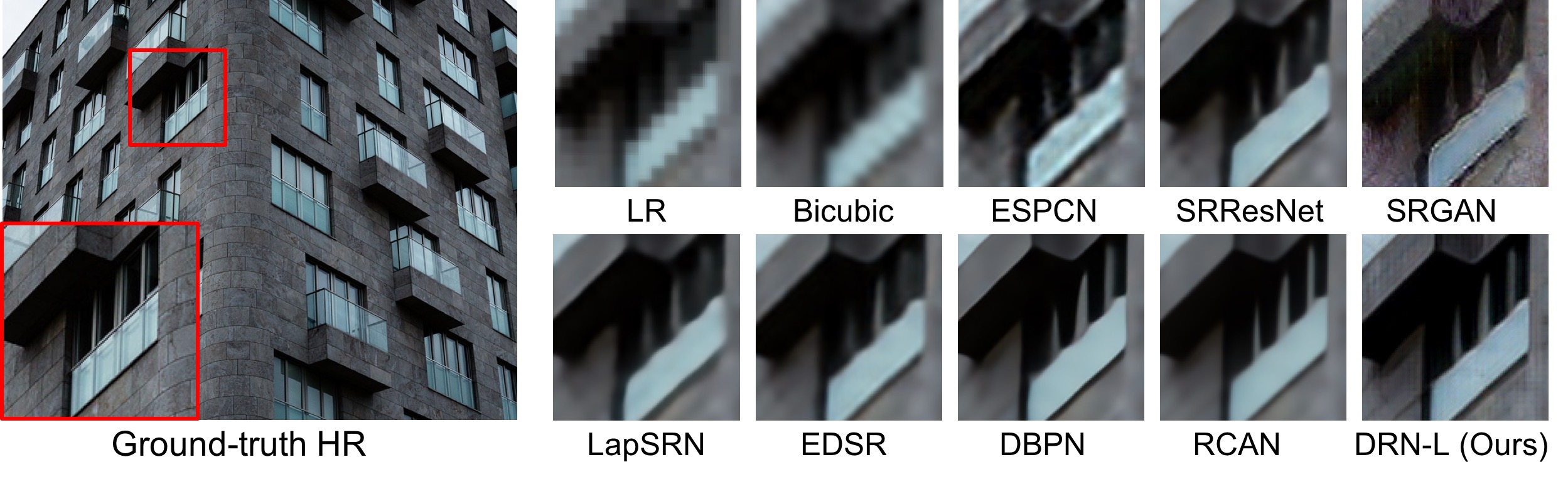}\label{fig:img8x_compare3_sup}
	}
	\caption{Visual comparison for $8\times$ image super-resolution on benchmark datasets.}
	\label{fig:image_compare_8x_sup}
\end{figure*}

\begin{figure*}[h]
	\renewcommand\thefigure{E}
	\centering
	\subfigure{
		\includegraphics[width = 0.47\columnwidth]{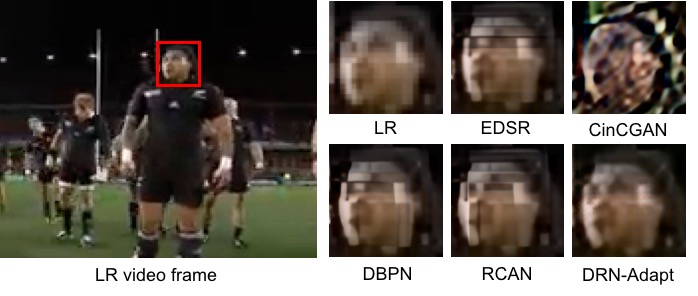}\label{fig:img4x_compare1}
	}
	\subfigure{
		\includegraphics[width = 0.47\columnwidth]{unpaired_video_2.jpg}\label{fig:img8x_compare1}
	}	
	\caption{Visual comparison of model adaptation for $8\times$ super-resolution on real-world video frames (from YouTube).}
	\label{fig:video_sup}
\end{figure*}

}

\end{document}